\documentclass{article}

\usepackage[margin=1in]{geometry}
\usepackage[round]{natbib}

\usepackage[utf8]{inputenc}
\usepackage[T1]{fontenc}
\usepackage{amsmath,amssymb,amsfonts,amsthm}
\usepackage{mathtools}
\usepackage{graphicx}
\usepackage{booktabs}
\usepackage{microtype}
\usepackage{xcolor}
\usepackage{url}
\usepackage{algorithm}
\usepackage{algpseudocode}
\usepackage{hyperref}
\renewcommand{\And}{\and}

\def\REV#1{{\color[HTML]{000000}#1}}

\newenvironment{REV*}
  {\begingroup\color[HTML]{000000}}
  {\endgroup}

\newtheorem{theorem}{Theorem}
\newtheorem{proposition}[theorem]{Proposition}
\newtheorem{lemma}[theorem]{Lemma}
\newtheorem{corollary}[theorem]{Corollary}
\newtheorem{assumption}[theorem]{Assumption}
\theoremstyle{definition}
\newtheorem{remark}[theorem]{Remark}

\newcommand{\R}{\mathbb{R}}
\newcommand{\C}{\mathbb{C}}
\newcommand{\N}{\mathbb{N}}
\newcommand{\din}{d_{\mathrm{in}}}
\newcommand{\dout}{d_{\mathrm{out}}}
\DeclareMathOperator{\rank}{rank}
\let\Im\relax
\DeclareMathOperator{\Im}{Im}
\DeclareMathOperator{\dist}{dist}

\title{Width-Independent Compressibility of Deep Neural Networks}

\author{Hong-Yi Wang \\
Princeton University \\
Princeton, NJ 08544 \\
\texttt{hywang@princeton.edu}
\And
Mingze Wang \\
Peking University \\
Beijing, China 100871 \\
\texttt{mingzewang.math@gmail.com}
\And
Liu Ziyin \\
Massachusetts Institute of Technology \\
Cambridge, MA 02139 \\
\texttt{ziyinl@mit.edu}
}

\begin{document}

\maketitle

\begin{abstract}
It has long been known that well-trained neural networks can be compressed very strongly without affecting their performance, an important phenomenon that remains poorly understood. We prove a uniform compressibility theorem for deep multilayer perceptrons with analytic activations. For a deep, wide \REV{fixed teacher network}, there exists a narrow (same depth) network that approximately represents the same function as the original. The reachable compressed width is strikingly independent of the original width, but is $O((\log(1/\varepsilon))^{d_{\mathrm{in}}})$, where $\varepsilon$ is the error budget and $d_{\mathrm{in}}$ is the \REV{effective input dimension}. Our construction involves a novel derivative-matching technique which is aware of the low-dimensional input, and a layer-wise reweighting that preserves the input-output mapping.
\end{abstract}

\section{Introduction}
\label{sec:intro}

Modern neural networks often behave as though only a fraction of their parameters are essential. At language-model scale, weights can be quantized to four bits, matrices can be narrowed by removing rows and columns, and entire layers can be dropped while largely preserving model capabilities \citep{frantar2023gptq,ashkboos2024slicegpt,men2024shortgpt,gromov2025unreasonable}. That such different interventions succeed points to a common phenomenon: the function learned by a wide network may admit a much smaller representation.

This observation raises a basic theoretical question. Given a particular trained deep network, can we construct a narrow network that approximates the same function \emph{uniformly over a bounded input domain}? If so, how does the size of the compressed network scale if the allowed error is limited? Such a guarantee would turn empirical redundancy into a property of the represented function itself.

Existing theory does not yet provide this combination. Compression guarantees typically control generalization, empirical error, or error under a calibration distribution, and often rely on data-dependent pruning or low-rank structure \citep{arora2018stronger,baykal2019data,liebenwein2020provable,zhang2025theoretical}. Uniform compression theorems for fixed networks are available only in more restricted settings, such as two-layer permutation-symmetric networks \citep{wang2026compression}. What is missing is a constructive, data-free sup-norm guarantee for a fixed deep MLP that composes across layers with controlled error.

We close this gap for analytic-activation deep MLPs. Fix a layer $\ell$ and write $f$ for the input-to-layer-($\ell{-}1$) map. The contribution of layer $\ell$ to the next pre-activation is
\begin{equation}
    y(x)=\sum_{i=1}^{d_\ell} v_i\,\sigma(w_i^\top f(x)),
    \label{eq:intro-decomp}
\end{equation}
a finite sum over the $d_\ell$ neurons of layer~$\ell$. If $f$ and $\sigma$ extend holomorphically to suitable complex neighborhoods, every summand is a holomorphic function of the input $x$ on a complex polydisk, and Cauchy's inequality converts finite-order \emph{input-derivative matching at the origin} into a uniform sup-norm bound on a bounded real domain $\Omega\subset\R^\din$ of radius $R$: matching all derivatives $\partial_x^\alpha y(0)$ for $|\alpha|\le k$ controls the Taylor remainder by a geometric factor $(R/\rho_0)^{k+1}$, where $\rho_0$ is determined by the activation's holomorphy radius. The matching condition is finite-dimensional: it imposes only $D_k=\binom{\din+k}{\din}$ scalar constraints per output coordinate. Therefore, a deterministic rank-reduction procedure on the derivative-feature matrix retains at most $D_k$ of the original neurons and recomputes the outgoing weights.

Solving for the smallest $k$ that meets a target accuracy $\varepsilon_\ell$ gives a compressed width
\begin{equation}
    d_{\ell}' = O\!\left(\frac{1}{\din!}
    \left(\frac{\log(2C_{0}/\varepsilon_\ell)}{\lambda}\right)^{\din}\right),
    \label{eq:intro-width}
\end{equation}
\begin{REV*}
where $\lambda$ is determined by the ratio between the input radius and the holomorphy radius, and $C_0$ is determined by the teacher network---mostly the range of weights. Eq.~\eqref{eq:intro-width} is width-independent in the sense that: after these teacher-specific constants are evaluated, the cap $D_{k_\ell^\star}$ has no additional explicit $d_\ell$ factor.

Composing layer-wise compressions in backward order keeps the analytic constants of the layer currently being compressed bounded by their original-teacher values, while downstream Lipschitz factors are absorbed by a uniform per-step error budget. This produces a same-depth compressibility theorem for the fixed teacher, with the per-layer bounds expressed similar to Eq.~\eqref{eq:intro-width} with replaced constants.
\end{REV*}

Conceptually, our theory establishes the following rather surprising idea:
\begin{quote}
    \textit{For a broad class of models (specified in the assumptions), the effective complexity of neural networks solely depends on the data dimension. }
\end{quote}

Our technical contributions that lead to this result are:
\begin{itemize}
    \item \REV{A fixed-teacher deep-MLP compression theorem (Theorem~\ref{thm:deep-comp-informal}): backward layer-wise compression of all $L-1$ hidden layers yields a uniformly accurate, same-depth network. For a fixed teacher network, its width caps are polylogarithmic in $1/\varepsilon$ with no additional explicit original-width factor.}
    \item The per-layer derivative-matching error bound that powers it (Theorem~\ref{thm:step1}): matching all input derivatives at the origin up to order $k$ controls the Taylor remainder by a geometric factor $(R/\rho_0)^{k+1}$ in terms of holomorphic radii of $f$ and $\sigma$.
    \item An explicit rank-reduction procedure (Algorithm~\ref{alg:compression}) that selects at most $\binom{\din+k}{\din}$ original neurons and recomputes outgoing weights while controlling the error from the network's input to output.
\end{itemize}

The remainder of this paper is organized as follows. Section~\ref{sec:setup} fixes notation, states the analyticity and weight-bound assumptions on the network, and formulates the per-layer compression problem. Section~\ref{sec:related} situates the result within approximation theory, depth--width separations, empirical and theoretical compression, and tensor-moment methods. Section~\ref{sec:proof-idea} presents the proof idea for compressing a single intermediate layer, building from a measure-compression reformulation through derivative matching to the rank-reduction step, and culminates in the single-layer compression theorem (\REV{Theorem~\ref{thm:main}}). Section~\ref{sec:deep} composes the single-layer result across all hidden layers of a deep MLP, derives the polylogarithmic-width deep theorem (Theorem~\ref{thm:deep-comp-informal}), and reports numerical experiments on a four-hidden-layer network. Section~\ref{sec:discussion} discusses conceptual implications, limitations, and open directions. \REV{Full proofs and supporting technical results are deferred to the appendix.}

\section{Problem setup}
\label{sec:setup}

Consider an $L$-layer neural network
\begin{equation}
    h^{(0)}=x,\qquad
    h^{(\ell)}=\sigma\!\bigl(W^{(\ell)}h^{(\ell-1)}\bigr)\quad(\ell=1,\dots,L-1),
    \qquad
    y_{\mathrm{out}}=W^{(L)}h^{(L-1)}\in\R^{\dout},
    \label{eq:mlp}
\end{equation}
where $x\in\Omega$ and $\Omega$ is a bounded subset of $\R^{\din}$, and $R\coloneqq\sup_{x\in\Omega}\|x\|$. Without loss of generality, we assume $0\in\Omega$. The dimension of the $\ell$th layer (i.e., width) is $d_\ell$, and in particular we denote $d_0=\din$ and $d_L=\dout$. Usually, there is a bias term $b^{(\ell)}$ for each layer in Eq.~\eqref{eq:mlp}. Here, we note that $b^{(\ell)}$ can be absorbed into $W^{(\ell)}$ (for notation simplicity) by appending a constant coordinate equal to $1$ to each hidden layer, converting affine maps into linear maps. 

We work in the regime $\din = O(1)$ and $d_1,\dots,d_{L-1} \gg 1$. We aim to compress the hidden layers while approximately preserving the input-to-output mapping. We write $F:\Omega\to\R^{\dout}$ for the input-to-output map of the original network. \REV{The goal is to construct a same-depth approximation $\widetilde F$ satisfying $\sup_{x\in\Omega}\|\widetilde F(x)-F(x)\|_\infty\le\varepsilon$ for a given error budget $\varepsilon>0$, with retained widths controlled by accuracy-driven dimension bounds.}

\REV{Throughout, $F$ is one fixed, finite teacher network. Unless uniform control of those quantities is imposed separately, our theorems do not assert a common retained-width bound for a family of teachers whose original widths vary.}

In what follows, we will first show the compression of a single intermediate layer, say the $\ell$-th layer, given an error budget $\varepsilon_\ell$, and then compose the compressions of all hidden layers and discuss the error budget assignment across layers. 

\begin{REV*}
\paragraph{Effective input dimension.}
For the layer under consideration, write $f=h^{(\ell-1)}$ for its upstream map. We call $m$ an \emph{admissible effective input dimension} if, after translating coordinates if necessary, there are a bounded set $U\subset\R^m$ with $0\in U$, a map $\pi:\Omega\to U$, and a map $\bar f:U\to\R^{d_{\ell-1}}$ such that $f=\bar f\circ\pi$, and the transformed problem obtained by replacing $(\Omega,x,\din,f)$ with $(U,u,m,\bar f)$ satisfies Assumptions~\ref{asm:holo}--\ref{asm:R-rho}. We denote the smallest admissible value by $m_{\mathrm{eff},\ell}$. The ambient choice $m=\din$, $U=\Omega$, and $\pi(x)=x$ is always the baseline; a smaller value requires an exact factorization together with the stated analytic and radius conditions in the reduced coordinates.
\end{REV*}

We will use the following assumptions, mostly for bounding the Taylor expansion remainder. $\ell\in [L-1]$ is any hidden layer, and $f:\Omega\to\R^{d_{\ell-1}}$ is the input-to-layer-($\ell{-}1$) map, i.e., $f(x)=h^{(\ell-1)}$.

\begin{assumption}[Holomorphy of $f$ and $\sigma$]
\label{asm:holo}
\leavevmode
\begin{itemize}
\item $f$ extends holomorphically to a complex ball $B_\C(0,\rho_f)\subset\C^\din$. For $0<r<\rho_f$, define
\begin{equation}
    L_f(r)\coloneqq\sup_{\|z\|\le r}\|Jf(z)\|_{\mathrm{op}}<\infty.
    \label{eq:Lf-def}
\end{equation}
\item $\sigma:\R\to\R$ is real analytic. Let $\mathcal D_\sigma\subset\C$ be the domain of its maximal holomorphic extension and, with $W_{\max}^{(\ell)}$ as in Assumption~\ref{asm:Wmax} below, define
\begin{equation}
    \rho_\sigma\coloneqq
    \inf_{\|w\|\le W_{\max}^{(\ell)}}
    \dist\!\bigl(w^\top f(0),\C\setminus\mathcal D_\sigma\bigr)>0.
    \label{eq:rhosigma}
\end{equation}
\end{itemize}
\end{assumption}
Real analytic activation functions are common, e.g., sigmoid, tanh, and GELU. In particular, GELU is holomorphic on the entire $\C$, so we expect our error bounds to work for any size of $\Omega$. This assumption does not hold for ReLU. \REV{Extending the construction to ReLU would require a separate approximation argument and is left for future work.}

\begin{assumption}[Per-layer weight bound]
\label{asm:Wmax}
The incoming weights at layer $\ell$ satisfy
\begin{equation}
    W_{\max}^{(\ell)}\coloneqq\max_{i\in[d_\ell]}\|w_i^{(\ell)}\|<\infty.
    \label{eq:Wmax-def}
\end{equation}
\end{assumption}

\begin{assumption}[Outgoing $\ell_1$-norm bound]
\label{asm:B2}
The outgoing weights at layer $\ell$ satisfy
\begin{equation}
    V_{\max}\coloneqq\max_{j\in[d_{\ell+1}]}\|W^{(\ell+1)}_{j,:}\|_1<\infty.
    \label{eq:Vmax-def}
\end{equation}
\end{assumption}

The effective holomorphic radius of $g(\cdot, w)$ is defined as
\begin{equation}
    \rho\coloneqq\sup\Bigl\{r\in(0,\rho_f):r\,W_{\max}^{(\ell)}\,L_f(r)<\rho_\sigma\Bigr\}.
    \label{eq:rho-eff}
\end{equation}

\begin{assumption}[Error-budget compatibility]
\label{asm:R-rho}
$R<\rho$.
\end{assumption}
Assumption~\ref{asm:R-rho} ensures that the MLP function from the input ($\forall x\in\Omega$) to $h^{(\ell+1)}$ is analytic.

\section{Related works}
\label{sec:related}

\paragraph{Approximation theory.}
Classical universal approximation theorems show that shallow networks of sufficient width approximate continuous functions on compact domains \citep{cybenko1989approximation,hornik1989multilayer,hornik1991approximation,pinkus1999approximation}. Quantitative refinements give explicit rates for Barron-type two-layer networks \citep{barron1993universal,weinan2021barron}, smooth and analytic targets \citep{mhaskar1996neural}, and shallow networks with $\mathrm{ReLU}^k$ activations \citep{siegel2020approximation,siegel2022high,mao2024approximation}; \citet{devore2021approximation} survey the approximation theory literature, including matching parameter-count bounds for Sobolev/Besov targets and the role of unstable parameterization. These works approximate an external target class. \REV{We instead approximate a fixed teacher network uniformly on a bounded domain, with an accuracy-driven retained-width bound.}
\REV{Kolmogorov-width and metric-entropy asymptotics for analytic function balls strongly indicate that the resulting $[\log(1/\varepsilon)]^m$ dependence, with $m$ the effective input dimension, is not an artifact of our algorithm but has information-theoretic grounds; Appendix~\ref{app:info-theory} makes this comparison precise.}

\paragraph{Depth, width, and smoothness.}
Depth improves approximation in Sobolev and related smoothness classes \citep{yarotsky2017error,yarotsky2018optimal,lu2021deep,schmidt2020nonparametric,liu2026smoothness}, while \emph{depth--width separations} show that certain functions require exponential width when depth is fixed \citep{telgarsky2016benefits,eldan2016power}, and that bounded-width networks recover universality only by growing depth \citep{lu2017expressive,kidger2020universal,hanin2019universal,vardi2022width}. \REV{Our theory preserves the original depth and gives each fixed covered teacher a same-depth approximation. The depth--width separation claims crucially rely on a family of increasingly hard functions, whereas our theory treats properties of a teacher network as fixed.}

\paragraph{Empirical compression at scale.}
Practical compression is dominated by four primitives: pruning, quantization, low-rank factorization, and distillation \citep{cheng2018survey,hoefler2021sparsity}, with foundational milestones in deep compression \citep{han2016deep}, lottery tickets \citep{frankle2019lottery}, and the methodological review of \citet{blalock2020state}. \REV{At LLM scale, post-training compression also spans structured row/column deletion \citep{ashkboos2024slicegpt}, depth pruning \citep{men2024shortgpt,gromov2025unreasonable}, and weight-quantization variants \citep{frantar2023gptq}.} These primitives are data-dependent and architecture-specific. They overlap with our work on application to post-training compression of a deep network without fine-tuning and thus provide motivation for this work. \REV{Our result is different in kind: it proves a uniform function-error certificate and a corresponding retained-width cap, not a claim of empirical superiority over these methods.}

\paragraph{Theoretical compression bounds and distillation.}
Existing constructive guarantees occupy a narrow strip: data-dependent coresets \citep{baykal2019data,liebenwein2020provable}, spectral/activation-covariance pruning \citep{suzuki2020spectral}, sparse-linear-approximation viewpoints \citep{yang2022theoretical}, NTK-preserving pruning of two-layer ReLU networks \citep{yang2023pruning}, low-rank recovery under approximately-low-rank activations \citep{zhang2025theoretical}, and randomized greedy schemes \citep{elcheairi2025theoretical}. Each is data-dependent, primitive-specific, or restricted to two-layer networks. Distillation \citep{hinton2015distilling,romero2015fitnets}, where a student is trained to mimic the teacher's function, is the closest practical cousin of derivative matching; theoretical accounts \citep{phuong2019towards,allenzhu2023towards} establish formal reductions in restricted settings (linear classifiers; multi-view data) but do not give explicit width bounds for deep nets, and the matching is on data rather than on the local Taylor expansion.

\paragraph{Tensor moment methods.}
Tensor-moment--like methods have been used to understand neural networks in several ways, see e.g., \citep{anandkumar2014tensor,janzamin2015beating}. 
% \citet{anandkumar2014tensor} establish that low-order observed-moment tensors recover a wide class of latent-variable models, and \citet{janzamin2015beating} use input-score-function tensors for guaranteed training of two-layer networks. 
The methodological predecessor of this work is \citet{wang2026compression}. It proves a polylogarithmic compression theorem for permutation-symmetric objects via tensor moment matching of the weight cloud. The present work advances that direction in three structural ways. First, the finite basis is the input-derivative basis at the origin, not a tensor-moment basis over weights. This resolves the problem of \citep{wang2026compression} being unable to compress a hidden layer whose previous or next layer is also wide. \REV{Second, for one fixed teacher our retained-width formula has no additional explicit $d_\ell$ factor, whereas the bound of \citet{wang2026compression} contains a polylogarithmic factor in the original width.} Third, here we establish compressibility of all layers of a deep network, whereas \citet{wang2026compression} treats a two-layer MLP with narrow input and output layers.

\section{Proof idea for compressing one layer}
\label{sec:proof-idea}

In this section, we focus on the following function defined by the $\ell$-th layer:
\begin{equation}
    \begin{aligned}
        x \mapsto y &= W^{(\ell+1)}\,\sigma\!\bigl(W^{(\ell)} h^{(\ell-1)}\bigr)\\
        &= \sum_{i=1}^{d_\ell} v_i\,\sigma\!\bigl(w_i^{\top} f(x)\bigr),
    \end{aligned}
\end{equation}
where $f:\Omega\to\R^{d_{\ell-1}}$ is the input-to-layer-($\ell{-}1$) map, $v_i$'s are the column vectors of $W^{(\ell+1)}$, and $w_i^\top$'s are the row vectors of $W^{(\ell)}$. We also write $g(x,w) \coloneqq \sigma(w^\top f(x))$. We aim to replace
\begin{equation}
    \begin{aligned}
        W^{(\ell)} &\in \R^{d_\ell \times d_{\ell-1}}
        \;\text{ by }\;
        \widetilde W^{(\ell)} \in \R^{d_\ell' \times d_{\ell-1}},\\
        W^{(\ell+1)} &\in \R^{d_{\ell+1} \times d_\ell}
        \;\text{ by }\;
        \widetilde W^{(\ell+1)} \in \R^{d_{\ell+1} \times d_\ell'} ,
    \end{aligned}
\end{equation}
where $d_{\ell}'$ is the compressed width. 

% The key idea is to match all input derivatives of $y$ at the origin up to some order $k$. If $f$ and $\sigma$ extend holomorphically to suitable complex neighborhoods, then every summand is a holomorphic function of the input on a complex polydisk, and Cauchy's inequality converts finite-order derivative matching at the origin into a uniform sup-norm bound on $\Omega$: matching all derivatives $\partial_x^\alpha y(0)$ for $|\alpha|\le k$ controls the Taylor remainder by a geometric factor $(R/\rho_0)^{k+1}$, where $\rho_0$ is determined by the activation's holomorphy radius.

\paragraph{Measure compression.}
We start by reformulating the $\ell$-th layer as an integration over the empirical measure of its neurons, the mean-field view of a network \citep{mei2018mean}. Define the counting measure $\mu=\sum_{i=1}^{d_\ell}\delta_{(v_i,w_i)}$ on $\R^{d_\ell'}\times\R^{d_{\ell-1}}$ over pairs of outgoing and incoming weights. For any function $a(v,w)$ and measure $\nu$,
\begin{equation}
    \langle a(v,w)\rangle_\nu\coloneqq\int a(v,w)\,d\nu(v,w).
\end{equation}
With this notation, the $\ell$-th layer can be written as $y(x) = \langle v g(x,w) \rangle_\mu$. The compression problem is then to find a new measure $\tilde{\mu}$ supported on $d'_\ell$ points, such that 
\begin{equation}
\label{eq:goal}
    \sup_{x\in\Omega}\bigl\|\langle v\,g(x,w)\rangle_{\tilde\mu}-\langle v\,g(x,w)\rangle_\mu\bigr\|_\infty\;\le\;\varepsilon_\ell,
\end{equation}
where $\varepsilon_\ell>0$ is the error budget allocated to this layer.

\paragraph{Derivative matching.} Eq.~\eqref{eq:goal} has to be satisfied for a large class of functions $y(x)$. One approach \citep{wang2026compression} is to match the expectation of low-degree polynomials of $v,w$, so that the error occurs only in the Taylor expansion tail (assuming any relevant function of $W^{(\ell)}$ is, in a sense, analytic, so that Taylor expansions converge to the actual function). \citet{wang2026compression} matches a monomial basis with $\binom{\din+\dout+k}{k}$ compressed points and proves that this is optimal if $y$'s dependence on $v$ and $w$ is a \emph{generic} permutation-symmetric function. However, here we surpass the bound by exploiting the specific expression of MLPs: it suffices to \emph{match the expectation of the derivatives of $g$, which has a much smaller basis} than the monomial basis. 

Note that we need to match the expectations $\langle v\,g(x,w)\rangle_\mu$---nothing more and nothing less, and that $g$ as a function of $w$ is parametrized by $O(1)$ parameters, i.e., by $x\in\Omega$. Assume $g$ is analytic in a sufficiently large range (see Assumption~\ref{asm:R-rho}) so that a Taylor expansion exists and converges to the actual function. From
\begin{equation}
    g(x,w) \approx \sum_{|\alpha|\le k} \frac{\partial_x^\alpha g(0,w)}{\alpha!} x^\alpha,
\end{equation}
(where $\alpha$ is a multi-index defined in Appendix~\ref{app:notation}), it follows that 
\begin{equation}
    \langle v\,g(x,w)\rangle_\mu \approx \sum_{|\alpha|\le k} \frac{x^\alpha}{\alpha!}  \langle v\,\partial_x^\alpha g(0,w)\rangle_\mu .
\end{equation}
Therefore, if we change a measure but preserve the expectations $\langle v\,\partial_x^\alpha g(0,w)\rangle$ for $|\alpha|\le k$, they represent almost the same input-output map. The compression error starts from the degree-$(k+1)$ Taylor remainder, which is thus of order $(R/\rho)^{k+1}$. This is summarized in the following theorem:
\begin{theorem}[Derivative-matching error]
\label{thm:step1}
Let $k\ge 0$, let $\tilde\mu=\sum_{i\in\mathcal K}\delta_{(\tilde v_{i},w_{i})}$ with $\mathcal K\subseteq [d_\ell]$, and write $\tilde y(x)\coloneqq\langle v\,g(x,w)\rangle_{\tilde\mu}=\sum_{i\in\mathcal K}\tilde v_{i}\,g(x,w_{i})$. Suppose
\begin{equation}
\label{eq:match}
    \bigl\langle v\,\partial_x^\alpha g(x,w)\big|_{x=0}\bigr\rangle_{\tilde\mu}\;=\;\bigl\langle v\,\partial_x^\alpha g(x,w)\big|_{x=0}\bigr\rangle_\mu
    \qquad\text{for every }\alpha\in\N^\din\text{ with }|\alpha|\le k,
\end{equation}
where both sides are $\R^{d_{\ell+1}}$-valued. Then for every $j\in[d_{\ell+1}]$ and every $x\in\Omega$,
\begin{equation}
\label{eq:step1-bound}
    \bigl|y_j(x)-\tilde y_j(x)\bigr|\;\le\;\Bigl(\|v_{:,j}\|_1+\|\tilde v_{:,j}\|_1\Bigr)\,M_g\,\frac{\rho_0}{\rho_0-R}\,\Bigl(\frac{R}{\rho_0}\Bigr)^{k+1},
\end{equation}
where $v_{:,j} \coloneqq (v_{ij})_{i\in[d_\ell]}\in\R^{d_\ell}$ and $\tilde v_{:,j} \coloneqq (\tilde v_{ij})_{i\in \mathcal{K}}\in\R^{|\mathcal K|}$. $M_g$ is a constant determined by $g$: pick any $\rho_0\in(R, \rho)$,
\begin{equation}
\label{eq:Mg-def}
    M_g\;\coloneqq\;\sup_{\substack{\|w\|\le W_{\max}^{(\ell)}\\ \|z\|\le\rho_0}}\bigl|g(z,w)\bigr| .
\end{equation}
\end{theorem}

From here on, we define the \emph{derivative feature map}:
\begin{equation}
\label{eq:Phi-def}
    \Phi(w)\;\coloneqq\;\bigl(\partial_x^\alpha g(x,w)\big|_{x=0}\bigr)_{|\alpha|\le k}\in\R^{D_k},\qquad D_k\; \coloneqq \;\binom{\din+k}{\din}. 
\end{equation}
Stacking over the $d_\ell$ incoming rows of $W^{(\ell)}$, 
\begin{equation}
\label{eq:Phi-stacked}
    \Phi(W^{(\ell)}) \;\coloneqq\; \
    \begin{pmatrix}
        \Phi(w_1) & \cdots & \Phi(w_{d_\ell})
    \end{pmatrix}
    \;\in\;\R^{D_k \times d_\ell}.
\end{equation}

\paragraph{Rank reduction and reweighting.}
The task of finding a compressed measure $\tilde{\mu}$ has now become a linear algebra problem: find a matrix $V\in \R^{d_{\ell+1}\times d_\ell}$ that satisfies $V\,\Phi(W^{(\ell)})^\top=W^{(\ell+1)}\,\Phi(W^{(\ell)})^\top$. The $i$-th column of $V$ is $\tilde{v}_i$, and the new measure is \REV{$\tilde{\mu}=\sum_{\substack{i\in[d_\ell]\\ \tilde{v}_i\neq 0}}\delta_{(\tilde{v}_i,w_i)}$}. \REV{It turns out there exists a $V$ with at most $D_k$ nonzero columns: we prove this in Lemma~\ref{lem:rank} and also propose Algorithm~\ref{alg:compression} to find such a $V$.} Furthermore, $V_{\max}$ can be guaranteed to be inflated by at most a factor of $D_k$, which could affect the next layer's error bound but by a controlled amount.

\paragraph{Main theorem and width bound.}
\begin{REV*}
Fix any $\rho_0\in(R,\rho)$, recall $M_g$ from Eq.~\eqref{eq:Mg-def}, and define the $k$-independent constants
\begin{equation}
\label{eq:C0-def}
    C_0\;\coloneqq \;V_{\max}\,M_g\,\frac{\rho_0}{\rho_0-R}\quad\text{and}\quad\lambda\;\coloneqq \;\log(\rho_0/R)\;>\;0.
\end{equation}
\REV{If $C_0=0$, the layer contribution is degenerate and the zero-output replacement has zero error. We therefore state the nontrivial case $C_0>0$ below.}
\end{REV*}

The whole approach above culminates in the following theorem for compressing a single hidden layer. In plain language, it answers the question: if there is an error budget $\varepsilon_\ell$, what matching order $k$ suffices? By solving $\mathrm{error}\sim (R/\rho_0)^{k+1}$, it suffices to take $k^* = O( \log(1/\varepsilon_\ell) )$, which gives the compressed width $D_{k^\star(\varepsilon_\ell)}=O((\log(1/\varepsilon_\ell))^{\din})$.

\begin{REV*}
\begin{theorem}[Intermediate-layer compression]
\label{thm:main}
\REV{Fix the teacher network and a hidden layer satisfying Assumptions~\ref{asm:holo}--\ref{asm:R-rho}. For any $\rho_0\in(R,\rho)$ with $C_0>0$ and $\lambda$ as in Eq.~\eqref{eq:C0-def}, any error budget $\varepsilon_\ell>0$, and any integer $k\ge 0$ satisfying}
\begin{equation}
\label{eq:k-cond}
    (k+1)\,\lambda\;\ge\;\log(2\,C_0/\varepsilon_\ell)\;+\;\din\,\log(k+\din)\;-\;\log(\din!),
\end{equation}
\REV{there exist an index set $\mathcal K=\{\kappa_1,\ldots,\kappa_n\}\subseteq[d_\ell]$ with $n=|\mathcal K|\le D_k$, where $D_k\coloneqq \binom{\din+k}{\din}$, and weight matrices $\widetilde W^{(\ell)}\in\R^{n\times d_{\ell-1}},\,\widetilde W^{(\ell+1)}\in\R^{d_{\ell+1}\times n}$ (constructed in Corollary~\ref{cor:l1}) satisfying $\widetilde W^{(\ell)}_{r,:}=W^{(\ell)}_{\kappa_r,:}$ for every $r\in[n]$, such that the compressed network $\tilde y(x)=\widetilde W^{(\ell+1)}\sigma(\widetilde W^{(\ell)}f(x))$ satisfies}
\begin{equation}
\label{eq:thm-bound}
    \sup_{x\in\Omega}\bigl\|y(x)-\tilde y(x)\bigr\|_\infty\;\le\;\varepsilon_\ell.
\end{equation}
\emph{Asymptotics.} Define
\begin{equation}
    k^\star(\varepsilon_\ell)\;\coloneqq \;\min\bigl\{k\in\N_0:k\text{ satisfies Eq.~\eqref{eq:k-cond}}\bigr\}.
\end{equation}
\REV{Then, as $\varepsilon_\ell\to 0^+$ with this teacher and $\din,\rho_0,R,V_{\max},M_g$ fixed,}
\begin{equation}
\label{eq:kstar-asymp}
    k^\star(\varepsilon_\ell)\;=\;\frac{\log(2C_0/\varepsilon_\ell)}{\lambda}\;+\;O\!\bigl(\log\log(1/\varepsilon_\ell)\bigr)
    \;=\;\Theta\!\bigl(\log(1/\varepsilon_\ell)\bigr),
\end{equation}
and consequently
\begin{equation}
\label{eq:dprime-asymp}
    \REV{D_{k^\star(\varepsilon_\ell)}\;=\;O\!\left(\frac{1}{\din!}\left(\frac{\log(2C_0/\varepsilon_\ell)}{\lambda}\right)^{\din}\right).}
\end{equation}
\end{theorem}
\begin{remark}[Width dependence and finite-scale compression]
With the fixed teacher constants and $\din$ held fixed, $k^\star(\varepsilon_\ell)=\Theta(\log(1/\varepsilon_\ell))$ and $D_{k^\star(\varepsilon_\ell)}=O((\log(1/\varepsilon_\ell))^{\din})$. If $M_g$ and the radius quantities are controlled uniformly across a family, dependence on $V_{\max}$ enters the cap only through $\log V_{\max}$; if that outgoing norm is also uniformly bounded, the cap is uniform in $d_\ell$. For this bound to give a genuinely narrower layer, the original width must satisfy $d_\ell>D_{k^\star(\varepsilon_\ell)}$. Appendix~\ref{app:nonvacuity} gives sharper finite-width bookkeeping, a closed-form sufficient threshold, and the realized-rank caveat.
\end{remark}

\paragraph{Corollary (exact analytic parameterization).}
Suppose $\Omega=\iota(U)$ for a one-to-one real-analytic parameterization $\iota:U\subset\R^m\to\R^\din$, and suppose the reparameterized upstream map $f\circ\iota$ satisfies Assumptions~\ref{asm:holo}--\ref{asm:R-rho} in the $u$-coordinates (after translation if needed). Applying Theorem~\ref{thm:main} to $u$ gives the same uniform error guarantee on $\Omega$ with $D_k=\binom{m+k}{m}$ and hence retained width $O((\log(1/\varepsilon_\ell))^m)$, with constants determined in the reparameterized coordinates.

\paragraph{Corollary (architectural bottleneck).}
Fix $j<\ell$ and use $u=h^{(j)}(x)$ as the input to the suffix from layer $j+1$ through layer $\ell$. If that suffix satisfies the transformed assumptions on $h^{(j)}(\Omega)$ after recentering, then $d_j$ is admissible and Theorem~\ref{thm:main} holds with $D_k=\binom{d_j+k}{d_j}$. Whenever the transformed assumptions hold for the candidate bottlenecks, one may therefore use the architectural bound
\begin{equation}
    m_{\mathrm{eff},\ell}\;\le\;m_\ell^{\mathrm{arch}}
    \;\coloneqq\;\min\{\din,d_1,\dots,d_{\ell-1}\}.
\end{equation}
\end{REV*}

\paragraph{Numerics.}
\REV{We present small-scale numerical experiments to test the construction and its predicted geometric error decay over a finite range of derivative orders and input radii.}

\begin{figure}[!t]
  \centering
  \includegraphics[width=0.7\textwidth]{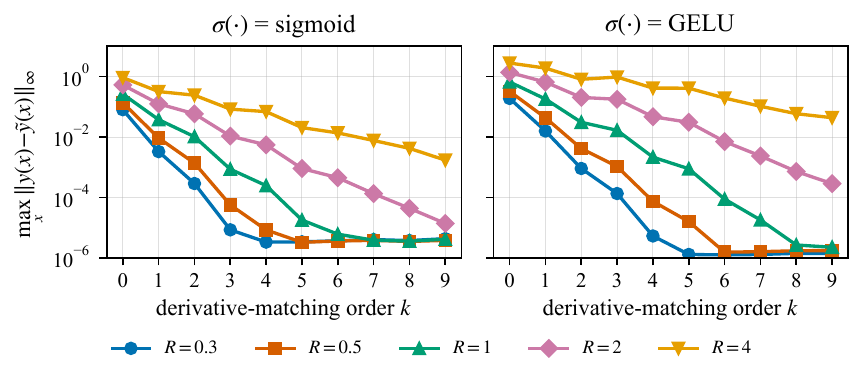}
  \caption{Single-layer compression error versus derivative-matching order $k$ on a randomly initialized $[3,2048,2048]$ network, for sigmoid (left) and GELU (right) activations. Each curve is one input radius $R$. Both panels exhibit geometric decay in $k$ at every $R$, with the per-step rate slowing as $R$ grows --- the qualitative behavior predicted by the $(R/\rho_0)^{k+1}$ rate of \REV{Theorem~\ref{thm:main}}. For $R\le 1$ the curves saturate at the float32 noise floor ($\sim\!10^{-6}$) by $k\approx 5$; at $k=9$ the error is below $10^{-5}$ uniformly in the unit ball for both activations.}
  \label{fig:onelayer-error}
\end{figure}

We instantiate \REV{Theorem~\ref{thm:main}} on a single hidden layer of a randomly initialized two-layer network with input dimension $\din=3$, hidden width $d_1=2048$, and output dimension $\dout=2048$. We compress the hidden layer from $2048$ neurons down to $D_k=\binom{3+k}{3}$ derivative-matched neurons. Two activations with qualitatively different properties are tested on identical random weights:
\begin{enumerate}
    \item sigmoid, which is real-analytic with poles at $\pm i\pi$ in $\C$, therefore $\rho_\sigma=\pi$;
    \item GELU, which is entire on $\C$ but whose $k$-th derivative at the origin grows superexponentially with $k$. 
\end{enumerate}
We measure $\sup_{x\in\mathcal X_R}\|y(x)-\tilde y(x)\|_\infty$ on a uniform sample from the closed ball $\{x\in\R^{\din}:\|x\|\le R\}$, matching the domain $\Omega$ in the statement of \REV{Theorem~\ref{thm:main}}, for different input radii.

\REV{Figure~\ref{fig:onelayer-error} is consistent with the qualitative finite-range prediction of Theorem~\ref{thm:main} for both sigmoid and GELU.} Both panels show clean geometric decay in $k$ at every radius, with the per-step rate slowing as $R$ grows --- consistent with the $(R/\rho_0)^{k+1}$ prediction.
% We do not draw an explicit theoretical slope on the figure because the relevant rate is set by the effective holomorphic radius $\rho_0$, which depends jointly on $\sigma$ and the weight matrix $W^{(1)}$, rather than by $\rho_\sigma$ alone (so a $(R/\pi)^{k+1}$ comparison for sigmoid would understate the true rate, and there is no pre-specified $\rho_\sigma$ for GELU). 
For $R\le 1$ the error reaches the float32 noise floor ($\sim\!10^{-6}$) by $k\approx 5$ for both activations; even at the largest tested radius $R=4$, which already exceeds $\rho_\sigma=\pi$ for sigmoid, the error still decreases exponentially over $k=0,\dots,9$. 
% At $k=9$ the layer has been compressed roughly $10\times$ ($2048\to 220$ neurons) and the error sits at the float32 epsilon for $R\le 1$.
As a remark, we would have expected that in the GELU case the error decays visibly faster than in the sigmoid case since $\rho_{\sigma} = \infty$ for $\sigma = \mathrm{GELU}$. However, \REV{the error decreases with $k$ for every tested $R$.} A possible explanation is that the higher derivatives of GELU grow rapidly at the origin, which hinders the error decay with $k$.

\section{Compressing all layers}
\label{sec:deep}

In this section, we discuss compression of all layers of a deep MLP. \REV{Applying Theorem~\ref{thm:main} to every hidden layer of one fixed teacher yields retained-width bounds that can be composed.} We need to address two additional questions here: (1) What is the optimal order to compress the layers? (2) If given a total error budget $\varepsilon$, how to assign it to each layer's compression? We will briefly discuss these caveats in the rest of this section, and leave a formal treatment to Appendix~\ref{app:composition}.

The short answers are: (1) \REV{Compressing the layers backward from the last hidden layer to the first keeps each step's analytic constants bounded by their original-network values} (Appendix~\ref{app:composition}, Proposition~\ref{prop:backward}), which licenses the budget split; forward order inflates these constants recursively (Proposition~\ref{prop:forward}) and yields no rigorous guarantee. Backward order is therefore preferred, though for moderately sized networks one is empirically free to choose any order without paying a measurable price. (2) When compressing in backward order, splitting the budget evenly as $\varepsilon_\ell = \varepsilon/\left((L-1)\,\overline\Lambda\right)$ --- where $\overline\Lambda\ge 1$ is an a-priori upper bound on the downstream Lipschitz factor, computed from the \emph{original} network (Appendix~\ref{app:composition}, Lemma~\ref{lem:apriori-lambda}) --- guarantees total error at most $\varepsilon$.

We provide intuitive reasoning for these facts. Per Sec.~\ref{sec:proof-idea}, after compressing the $\ell$-th layer, $W^{(\ell)}$ becomes a subset of the original, whereas $W^{(\ell+1)}$ is reweighted, and its maximum row $\ell_1$-norm (equivalently its $\ell_\infty\!\to\!\ell_\infty$ operator norm) increases at most by a factor of $D_k$. If we compress backward, the inflated $W^{(\ell+1)}$ is never compressed again in the subsequent compressions, so the error bounds of the layers essentially do not influence each other. However, if we compress forward, the inflated $W^{(\ell+1)}$ should be used to calculate the effective holomorphic radius at the next layer, which in the worst case shrinks $\rho$ by $D_k$---and this effect stacks as we compress deeper layers. The practical reason why the error of forward-order compression does not explode (see Fig.~\ref{fig:deep-error}) is presumably that the function $f: x\to h^{(\ell+2)}$ is not changed by much due to the compression of the $\ell$-th layer, even if $W^{(\ell+1)}$ is inflated. Nevertheless, this does not lead to a rigorous error guarantee for forward-order compression. 

Our main theoretical error bound for deep MLP compression is as follows, rigorously stated as Theorem~\ref{thm:deep-comp}. 

\begin{theorem}[Deep MLP compression; informal]
    \label{thm:deep-comp-informal}
    For any error budget $\varepsilon>0$, compressing the hidden layers of an $L$-layer MLP $F$ in backward order with a uniform per-layer budget yields a compressed network $\widetilde F$ satisfying
    \begin{equation}
        \sup_{x\in\Omega}\bigl\|\widetilde F(x)-F(x)\bigr\|_\infty\;\le\;\varepsilon,
    \end{equation}
    in which every hidden layer has compressed width
    \begin{equation}
        d_\ell'\;=\;O\!\left(\bigl(\log(L/\varepsilon)\bigr)^{\din}\right),
        \qquad \ell=1,\dots,L-1,
    \end{equation}
    where the suppressed constants depend only on the activation $\sigma$ and on the weight matrices of the original network (through the per-layer holomorphic radii, the downstream Lipschitz factors, and the constants in Theorem~\ref{thm:main}). \REV{Consequently, for fixed $L,d_{\mathrm{in}},d_{\mathrm{out}}$ and the fixed teacher's constants, the total parameter count of $\widetilde F$ is polylogarithmic in $1/\varepsilon$. Equation~\eqref{eq:param-count} records the explicit dependence on $L$.}
\end{theorem}

\begin{REV*}
\begin{figure}[!th]
  \centering
  \includegraphics[width=0.7\textwidth]{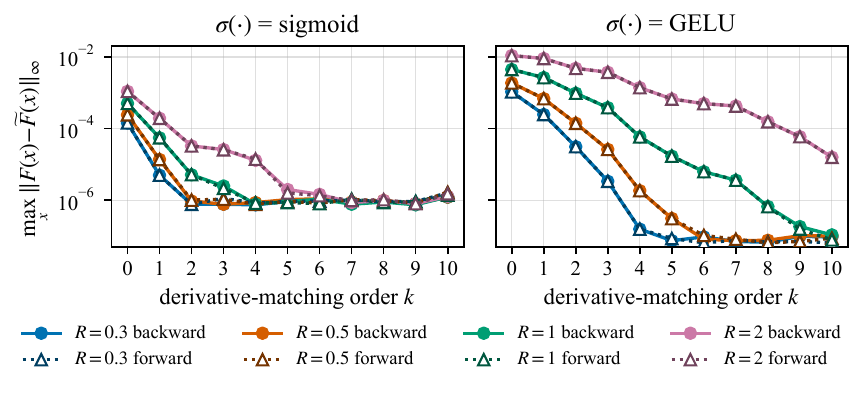}
  \caption{Composed compression error versus per-step derivative-matching order $k$ on a randomly initialized $[2,1024,1024,1024,1024,2]$ network, for sigmoid (left) and GELU (right). Each panel shows eight curves: four input radii $R\in\{0.3,0.5,1,2\}$ in four colors, with backward order (compress the deepest hidden layer first) drawn solid with circles and forward order (shallowest first) drawn dotted with triangles. The two orders are essentially indistinguishable at every $(k,R)$. Each compression step reduces its hidden layer to $D_k=\binom{2+k}{2}\in\{1,3,6,\dots,66\}$ neurons.}
  \label{fig:deep-error}
\end{figure}
\end{REV*}

\paragraph{Numerics.}
We instantiate the composition theorem on a deep network with $L=5$, four hidden layers of width $1024$, $\din=\dout=2$. Compression is applied to all four hidden layers sequentially with a uniform per-step order $k$, in two compression orders: backward $=[3,2,1,0]$ and forward $=[0,1,2,3]$. The two orders are applied to the same initialized weights for two activations (sigmoid and GELU). Each step compresses its hidden layer from $1024$ to $D_k=\binom{2+k}{2}$ neurons (from 1024 down to 66 at $k=10$). We measure $\sup_{x\in\mathcal X_R}\|F(x)-\widetilde F(x)\|_\infty$ on a uniform sample from the closed ball of radius $R$.

Figure~\ref{fig:deep-error} reports the resulting test error. The most striking empirical observation is that backward and forward orders are indistinguishable: at every $(k, R, \sigma)$ in the sweep the two compressed networks agree to within float32 noise, even though their per-step constants ($W_{\max}$ inflation, downstream Lipschitz factors) differ in theory. Both panels exhibit clean geometric decay in $k$ at every radius, with the per-step rate slowing as $R$ grows --- the same $(R/\rho_0)^{k+1}$-type behavior as in the single-layer experiment, now compounded across four hidden layers. Sigmoid reaches the float32 noise floor by $k\approx 2$ for $R\le 1$ and by $k\approx 5$ for $R=2$; GELU keeps decreasing through $k=10$, ending below $10^{-7}$ for $R\le 1$ and at $\sim\!10^{-5}$ for $R=2$. At the deepest sweep point ($k=10$) every hidden layer has been compressed from $1024$ to $66$ neurons --- a $\sim\!15\times$ width reduction per layer, two orders of magnitude in per-layer parameter count --- while the input--output map agrees with the original to noise-floor accuracy on the unit ball, \REV{which is consistent with the predicted geometric error decay over the tested finite range.}

\begin{REV*}
\paragraph{Post-training compression.}
In Fig.~\ref{fig:trained-deep}, we test the same construction after unconstrained training. On the unit disk in $\R^2$, five teacher networks with sigmoid activation are fit to
\begin{equation}
    q(x_1,x_2)=\sin(2x_1)+\tfrac12\cos(3x_2)+\tfrac14x_1x_2.
    \label{eq:trained-target}
\end{equation}
We use $20{,}000$ Sobol training points and evaluate on $150{,}000$ held-out Sobol points followed by projected local searches starting from the largest observed discrepancies. For each seed, the frozen randomly initialized network is compressed by the identical procedure as a control. Matching orders $k=0,\ldots,10$ give $D_k=\binom{k+2}{2}$.

\begin{figure}[!th]
  \centering
  \includegraphics[width=0.85\textwidth]{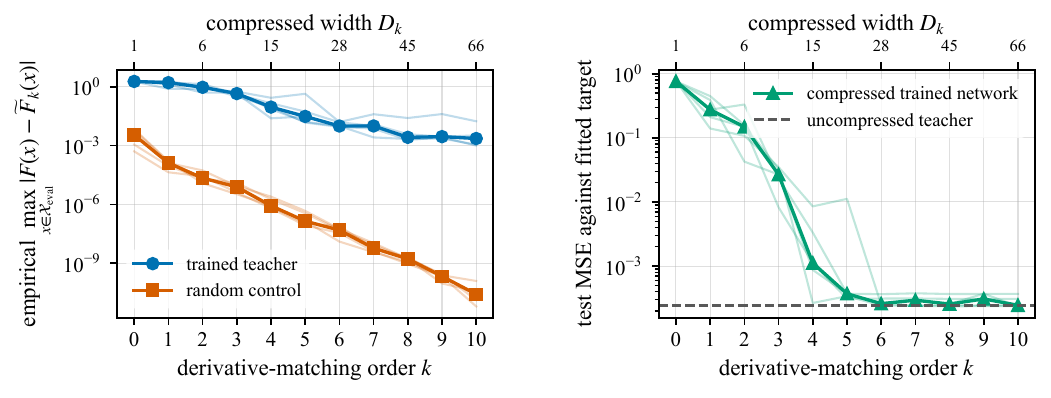}
  \caption{Compression of all three hidden layers of a trained $[2,512,512,512,1]$ sigmoid MLP, without fine-tuning. Thin curves are the five seeds and bold curves their medians. Left: empirical maximum teacher--compression discrepancy on the held-out set and projected local searches; frozen random controls are shown for comparison. Right: target MSE, with the median uncompressed-teacher MSE dashed.}
  \label{fig:trained-deep}
\end{figure}

We use a predeclared task-preservation criterion: the compressed network's target MSE be at most $1.05$ times its teacher's MSE. Across the five seeds, the median first width meeting this criterion is $28$, corresponding to a $304\times$ reduction in total parameter count. Train, validation, and test errors agree closely across all five seeds, with test/train MSE ratios in $[0.99953,0.99980]$, giving no sign of overfitting. Assumption~\ref{asm:R-rho} is violated for all trials (especially at layer $2$, $R/\rho$ is at least $3$), but the compression construction still yields a practical width reduction whose error decreases substantially over the tested range as $k$ increases. This further supports that the applicability of the construction is broader than the theorem's formal guarantee. Appendix~\ref{app:experiments} gives the full protocol, per-seed results, and numerical checks.
\end{REV*}

\section{Discussion and outlook}
\label{sec:discussion}

\begin{REV*}
We have established a uniform sup-norm approximation theorem for analytic deep MLPs based on input-derivative matching. At layer~$\ell$, the construction retains $d_\ell'\le D_{k_\ell^\star}$ neurons. For the fixed teacher, $D_{k_\ell^\star}$ is polylogarithmic in the inverse per-layer error. The certificate therefore has no additional explicit original-width factor once these constants are evaluated, but it can depend indirectly on width through the teacher's norms and downstream factors. The construction is explicit: compute derivative features, select a rank basis, and recompute the outgoing weights. To our knowledge, this is the first uniform compression theorem that applies to arbitrary intermediate layers of a deep MLP and composes layer by layer with a controlled error budget. It supplies a function-level existence certificate complementary to practical post-training methods.
\end{REV*}

Our construction (and associated numerical results) serves as a proof of existence and is not necessarily practical as-is. In particular, it is compute- and memory-expensive to evaluate the high-order derivative features. We do expect future work to improve on the residual error (expectedly by a constant factor) and to reduce the computational cost of compression. 

\begin{REV*}
The ambient feature count $\binom{\din+k}{\din}$ can be replaced by $\binom{m+k}{m}$ for any admissible effective input dimension $m$. Extending this result to approximately low-dimensional or noisy near-manifold inputs is an important open problem that would bring the theory closer to settings relevant to practical post-training compression. Empirical intrinsic-dimension estimates for common image datasets lie roughly between $10$ and $50$ \citep{pope2021intrinsic}. These estimates do not by themselves provide an admissible effective dimension under our theorem, but they illustrate the potential scale: if we match up to the fifth derivatives, the retained width ranges from $10^3$ to $10^6$. Thus, a successful near-manifold extension could yield retained-width bounds relevant to large-model compression.
\end{REV*}

Finally, the construction has several conceptual implications, all of which are worth exploring in future work:
\begin{enumerate}
    \item \REV{\emph{Critical retained-width.} Whether a network class admits a universal critical width, or a matching information-theoretic lower bound, remains open.}
    \item \emph{Token granularity.} There has been a tendency to train LLMs with low precision $\varepsilon$ to gain speed. However, our theory (also Figs.~\ref{fig:onelayer-error} and \ref{fig:deep-error} visually) suggests that wide networks do not store more patterns than a $\operatorname{polylog}(1/\varepsilon)$-wide network. It is thus worth exploring the concrete tradeoff between model expressivity and numerical precision.
    \item \emph{Proof of Lottery ticket hypothesis.} The existence of universal compression approaches could lead to a proof of the dynamical lottery ticket hypothesis (DLTH), as elucidated by \citet{wang2026compression}. Our compression method does not establish DLTH, but its derivative-matching construction may suggest tools for studying DLTH in deep neural networks.
\end{enumerate}

\label{sec:main-end}
\clearpage

\bibliographystyle{plainnat}
\bibliography{references}

\appendix

\section{Notations}
\label{app:notation}

\paragraph{Multi-index calculus.}
For a multivariate function on $\R^{\din}$ we use the standard multi-index shorthand. A \emph{multi-index} is a tuple $\alpha=(\alpha_1,\dots,\alpha_{\din})\in\N_0^{\din}$ of non-negative integers; its \emph{order} is $|\alpha|\coloneqq\alpha_1+\cdots+\alpha_{\din}$, its \emph{factorial} is $\alpha!\coloneqq\alpha_1!\,\alpha_2!\cdots\alpha_{\din}!$, the \emph{monomial} of $x\in\R^{\din}$ is $x^\alpha\coloneqq x_1^{\alpha_1}\cdots x_{\din}^{\alpha_{\din}}$, and the \emph{partial derivative} of order $\alpha$ is
\begin{equation}
    \partial_x^\alpha
    \coloneqq
    \frac{\partial^{|\alpha|}}{\partial x_1^{\alpha_1}\cdots\partial x_{\din}^{\alpha_{\din}}} .
\end{equation}
The Taylor polynomial of degree $k$ of a smooth function $h:\R^{\din}\to\R$ at the origin is $T_k h(x)\coloneqq\sum_{|\alpha|\le k} (x^\alpha/\alpha!)\,\partial_x^\alpha h(0)$. The number of multi-indices with $|\alpha|\le k$ in dimension $\din$ is
\begin{equation}
    D_k
    \;\coloneqq\;
    \#\{\alpha\in\N_0^{\din}:|\alpha|\le k\}
    \;=\;
    \binom{\din+k}{\din},
\end{equation}
which is the dimension of the space of polynomials of degree at most $k$ in $\din$ variables and the size of the basis we match.

\paragraph{Norms.}
For a vector $u\in\R^d$, $\|u\|$ denotes the Euclidean ($\ell_2$) norm, $\|u\|_1=\sum_i|u_i|$ the $\ell_1$ norm, and $\|u\|_\infty=\max_i|u_i|$ the sup-norm. For a matrix $A$, $\|A\|_{\mathrm{op}}$ is the spectral (operator-2) norm and $\|A\|_{\mathrm{op},\infty}$ is the $\ell_\infty\!\to\!\ell_\infty$ induced norm (the maximum absolute row sum). 

\paragraph{Matrix and submatrix notation.}
We use Python/NumPy-style slicing throughout. For a matrix $A\in\R^{m\times n}$, $A_{i,:}$ is its $i$-th row (as a vector), $A_{:,j}$ is its $j$-th column, and for index sets $I\subseteq[m]$, $J\subseteq[n]$, $A[I,J]$ is the submatrix with rows $I$ and columns $J$. We write $A^\top$ for the transpose and $A^+$ for the Moore--Penrose pseudoinverse. The image (column space) of $A$ is $\Im(A)$. The standard basis vector with a $1$ in coordinate $r$ and zeros elsewhere is $e_r$. 

\paragraph{Complex-analytic objects.}
The open complex ball of radius $r$ centered at the origin in $\C^d$ is $B_\C(0,r)\coloneqq\{z\in\C^d:\|z\|<r\}$, and $\overline{B_\C(0,r)}$ is its closure. For $z\in\C$ and a set $S\subseteq\C$, $\dist(z,S)\coloneqq\inf_{w\in S}|z-w|$. The Jacobian of $f$ at $z$ is denoted $Jf(z)$. Several holomorphic radii play distinct roles and are easy to confuse:
\begin{itemize}
    \item $\rho_f$: holomorphy radius of $f$ in the input variable (Assumption~\ref{asm:holo});
    \item $\rho_\sigma$: distance from the activation's input range to the boundary of $\sigma$'s holomorphic domain $\mathcal D_\sigma$ (Eq.~\ref{eq:rhosigma});
    \item $\rho$: the \emph{effective} holomorphic radius of the composed map $g(\cdot,w)$, defined in Eq.~\eqref{eq:rho-eff};
    \item $\rho_0$: a working radius chosen in $(R,\rho)$, used to apply Cauchy's inequality.
\end{itemize}
The hierarchy $R<\rho_0<\rho\le\rho_f$ is maintained throughout (Assumption~\ref{asm:R-rho}).

\paragraph{Measure-theoretic notation.} The Dirac function at $(v,w)$ is $\delta_{(v,w)}$.

\section{Formal theory on single-layer compression}
\label{app:single-layer}

\subsection{Derivative matching controls the output}

The goal of this subsection is to show that if the compressed measure $\tilde\mu$ matches all $x$-derivatives of $\langle v\,g(\cdot,w)\rangle_\mu$ at $0$ up to order $k$, then $\langle v\,g(x,w)\rangle_{\tilde\mu}$ approximates $\langle v\,g(x,w)\rangle_\mu$ uniformly on $\Omega$, with error decaying geometrically in $k$.

\begin{lemma}[Sufficient holomorphic radius for $g(\cdot,w)$]
\label{lem:holo}
Under Assumptions~\ref{asm:holo}--\ref{asm:R-rho}, for every $w\in\R^{d_{\ell-1}}$ with $\|w\|\le W_{\max}^{(\ell)}$ and every $r\in(0,\rho)$, the map $z\mapsto g(z,w)$ extends holomorphically to $B_\C(0,r)$. In particular, $g(\cdot,w)$ is holomorphic on an open neighborhood of $\overline{B_\C(0,\rho_0)}$.
\end{lemma}

\begin{proof}
Fix $r\in(0,\rho)$. Since $L_f$ is nondecreasing on $(0,\rho_f)$, the map $r'\mapsto r'W_{\max}^{(\ell)}L_f(r')$ is nondecreasing. The set in Eq.~\eqref{eq:rho-eff} is therefore down-closed; it is nonempty because $r'W_{\max}^{(\ell)}L_f(r')\to 0$ as $r'\to 0^+$ while $\rho_\sigma>0$, so $\rho$ is well defined, and down-closedness gives that the set contains all of $(0,\rho)$. In particular, $rW_{\max}^{(\ell)}L_f(r)<\rho_\sigma$. For $z\in B_\C(0,r)$, holomorphy of $f$ and the fundamental theorem of calculus along $t\mapsto tz$ give
\begin{equation}
\begin{aligned}
    f(z)-f(0)&=\int_0^1 Jf(tz)\,z\,dt,\\
    \|f(z)-f(0)\|&\le rL_f(r).
\end{aligned}
\label{eq:fz-fzero}
\end{equation}
Therefore
\begin{equation}
\begin{aligned}
    |w^\top f(z)-w^\top f(0)|
    &\le \|w\|\,\|f(z)-f(0)\|\\
    &\le W_{\max}^{(\ell)}rL_f(r)<\rho_\sigma.
\end{aligned}
\label{eq:wfz-bound}
\end{equation}
Thus $w^\top f(z)$ lies in the disk on which $\sigma$ is holomorphic. The composition $\sigma(w^\top f(z))$ is holomorphic on $B_\C(0,r)$. Since $\rho_0<\rho$, choosing $r_1\in(\rho_0,\rho)$ gives holomorphy near $\overline{B_\C(0,\rho_0)}$.
\end{proof}

\begin{proof}[Proof of Theorem~\ref{thm:step1}]
Fix $j$ and $x\in\Omega$, set $s=\|x\|\le R$, and assume $s>0$; the case $s=0$ follows from Eq.~\eqref{eq:match} with $|\alpha|=0$. Choose $r_1\in(\rho_0,\rho)$. For fixed $w$, define $\varphi_w(t)=g(tx,w)$ for $|t|<r_1/s$. Lemma~\ref{lem:holo} makes $\varphi_w$ holomorphic on a disk containing $|t|\le\rho_0/s$. Its Taylor expansion along the complex line $tx$ is
\begin{equation}
    \varphi_w(t)=\sum_{n=0}^\infty a_n(w)t^n,\qquad
    a_n(w)=\sum_{|\alpha|=n}\frac{x^\alpha}{\alpha!}\partial_x^\alpha g(0,w).
    \label{eq:phi-expansion}
\end{equation}
Indeed, by the chain rule, $a_n(w)=\frac{1}{n!}\frac{d^n}{dt^n}g(tx,w)\big|_{t=0}=\frac{1}{n!}\bigl[(x\cdot\nabla_x)^n g\bigr](0,w)$, and expanding $(x\cdot\nabla_x)^n$ by the multinomial theorem gives Eq.~\eqref{eq:phi-expansion}; grouping by $|\alpha|=n$ likewise gives $\sum_{n=0}^{k}a_n(w)=T_kg(x,w)$, where $T_k g(x,w) \coloneqq \sum_{|\alpha|\le k}x^\alpha\partial_x^\alpha g(0,w)/\alpha!$.
On $|t|=\rho_0/s$, $|g(tx,w)|\le M_g$, so Cauchy's inequality gives
\begin{equation}
    |a_n(w)|\le M_g(s/\rho_0)^n.
    \label{eq:an-cauchy}
\end{equation}
Since $s<\rho_0$,
\begin{equation}
    |g(x,w)-T_kg(x,w)|
    \le
    M_g\sum_{n\ge k+1}(s/\rho_0)^n
    \le
    M_g\frac{\rho_0}{\rho_0-R}
    \left(\frac{R}{\rho_0}\right)^{k+1}
    =: E_k.
    \label{eq:single-neuron-tail}
\end{equation}
Summing the tail bound over neurons yields
\begin{equation}
\begin{aligned}
    |y_j(x)-T_ky_j(x)|&\le\|v_{:,j}\|_1E_k,\\
    |\tilde y_j(x)-T_k\tilde y_j(x)|&\le\|\tilde v_{:,j}\|_1E_k.
\end{aligned}
\label{eq:Tk-tail-bounds}
\end{equation}
Equation~\eqref{eq:match} implies $T_ky_j(x)=T_k\tilde y_j(x)$ after multiplying each matched derivative by $x^\alpha/\alpha!$ and summing over $|\alpha|\le k$. The triangle inequality gives Eq.~\eqref{eq:step1-bound}.
\end{proof}

\subsection{Existence of compressed weights}
This subsection shows that, starting from the matrix equation $V\Phi(W^{(\ell)})^\top=W^{(\ell+1)}\Phi(W^{(\ell)})^\top$, there exists a solution $V$ with \REV{at most $D_k$ nonzero columns}, which gives the compressed weights. The key is to apply a rank-reduction argument to the derivative-feature matrix $\Phi(W^{(\ell)})$.

\begin{lemma}[Rank reduction]
\label{lem:rank}
Let $Z=(z_1|\cdots|z_m)\in\R^{D\times m}$ have rank $n$. Then there are distinct indices $\kappa_1,\dots,\kappa_n\in[m]$ and a matrix $C\in\R^{n\times m}$ such that, with $\mathcal K=\{\kappa_1,\dots,\kappa_n\}$ and $Z_{\mathcal K}=(z_{\kappa_1}|\cdots|z_{\kappa_n})$,
\begin{equation}
    Z=Z_{\mathcal K}C,\qquad C_{:,\kappa_r}=e_r\quad(r\in[n]),
    \label{eq:rank-factor}
\end{equation}
and
\begin{equation}
    \max_i\|C_{:,i}\|_1\le n.
    \label{eq:C-l1}
\end{equation}
If $n=0$, this means $\mathcal K=\varnothing$, $C\in\R^{0\times m}$, and $Z=0$.
\end{lemma}

\begin{proof}
The case $n=0$ is immediate. For $n\ge1$, choose row and column sets $I\subseteq[D]$ and $\mathcal K=\{\kappa_1,\dots,\kappa_n\}\subseteq[m]$ with $|I|=|\mathcal K|=n$ such that $A\coloneqq Z[I,\mathcal K]$ is nonsingular and maximizes $|\det A|$ among all nonsingular $n\times n$ submatrices. Then $Z_{\mathcal K}$ has full column rank and $\Im(Z_{\mathcal K})=\Im(Z)$. Set $C\coloneqq Z_{\mathcal K}^+Z$. Since $Z_{\mathcal K}Z_{\mathcal K}^+$ is the projector onto $\Im(Z)$, $Z=Z_{\mathcal K}C$, and selected columns satisfy $C_{:,\kappa_r}=e_r$.

For $i\notin\mathcal K$, the vector $c_i=C_{:,i}$ is the unique solution of $Z_{\mathcal K}c_i=z_i$. Restricting to rows $I$ gives $Ac_i=Z[I,i]$. By Cramer's rule,
\begin{equation}
    (c_i)_r=\frac{\det(A^{(r\leftarrow i)})}{\det A},
    \label{eq:cramer}
\end{equation}
where $A^{(r\leftarrow i)}$ replaces the $r$-th column of $A$ by $Z[I,i]$. The maximum-volume choice gives $|\det(A^{(r\leftarrow i)})|\le|\det A|$, so $|(c_i)_r|\le1$ for every $r$ and $\|c_i\|_1\le n$. For selected columns the bound is $\|e_r\|_1=1\le n$.
\end{proof}

\begin{proposition}[Existence of compressed weights]
\label{prop:compressed-V}
There exist $\mathcal K=\{\kappa_1,\dots,\kappa_n\}\subseteq[d_\ell]$ with \REV{$n\le D_k$} and a matrix $V\in\R^{d_{\ell+1}\times d_\ell}$ whose columns outside $\mathcal K$ are zero such that
\begin{equation}
    V\Phi(W^{(\ell)})^\top
    =
    W^{(\ell+1)}\Phi(W^{(\ell)})^\top.
    \label{eq:V-identity}
\end{equation}
Equivalently, the compressed measure $\tilde\mu=\sum_{r=1}^n\delta_{(\tilde v_{\kappa_r},w_{\kappa_r})}$ with $\tilde v_{\kappa_r}\coloneqq V_{:,\kappa_r}$ satisfies the derivative-matching condition Eq.~\eqref{eq:match}.
\end{proposition}

\begin{proof}
Apply Lemma~\ref{lem:rank} to $Z=[\Phi(w_1)|\cdots|\Phi(w_{d_\ell})]$. For each output row $j$, let $a^{(j)}\in\R^{d_\ell}$ be the $j$-th row of $W^{(\ell+1)}$ as a column vector and set $\tilde a^{(j)}=Ca^{(j)}\in\R^n$. Define $V_{j,\kappa_r}=\tilde a^{(j)}_r$ and $V_{j,i}=0$ for $i\notin\mathcal K$. Then the $j$-th row of $V\Phi(W^{(\ell)})^\top$ equals $(Z_{\mathcal K}Ca^{(j)})^\top=(Za^{(j)})^\top$, which is the $j$-th row of $W^{(\ell+1)}\Phi(W^{(\ell)})^\top$. Finally, Eq.~\eqref{eq:V-identity} coincides entrywise with Eq.~\eqref{eq:match}: for each $j\in[d_{\ell+1}]$ and each $|\alpha|\le k$, the corresponding entries of the two sides are $\sum_{i=1}^{d_\ell}V_{ji}\,\partial_x^\alpha g(0,w_i)=\bigl(\langle v\,\partial_x^\alpha g(0,w)\rangle_{\tilde\mu}\bigr)_j$ and $\sum_{i=1}^{d_\ell}W^{(\ell+1)}_{ji}\,\partial_x^\alpha g(0,w_i)=\bigl(\langle v\,\partial_x^\alpha g(0,w)\rangle_{\mu}\bigr)_j$.
\end{proof}

\begin{corollary}[Norm inflation due to compression]
\label{cor:l1}
With $\mathcal K,V$ as in Proposition~\ref{prop:compressed-V}, define $\widetilde W^{(\ell)}_{r,:}=w_{\kappa_r}^\top$ and $\widetilde W^{(\ell+1)}_{:,r}=V_{:,\kappa_r}$. Then the compressed network has width \REV{$n\le D_k$} and
\begin{equation}
    \|\widetilde W^{(\ell+1)}_{j,:}\|_1\le nV_{\max}\le D_kV_{\max}
    \qquad(j\in[d_{\ell+1}]).
    \label{eq:tilde-B2}
\end{equation}
\end{corollary}

\begin{proof}
The width bound is Proposition~\ref{prop:compressed-V}. For the norm bound, $\widetilde W^{(\ell+1)}_{j,:}=Ca^{(j)}$, so
\begin{equation}
\begin{aligned}
    \|Ca^{(j)}\|_1
    &\le\sum_{i=1}^{d_\ell}|a_i^{(j)}|\,\|C_{:,i}\|_1\\
    &\le n\|a^{(j)}\|_1
    \le nV_{\max}.
\end{aligned}
\label{eq:Cajbound}
\end{equation}
\end{proof}

\subsection{Rank reduction algorithm}

Algorithm~\ref{alg:compression} below turns the proof of Proposition~\ref{prop:compressed-V} into an explicit procedure.

\begin{algorithm}
\caption{Intermediate-layer compression via rank reduction}\label{alg:compression}
\begin{algorithmic}[1]
\Require Weight matrix $W^{(\ell)}\in\R^{d_\ell\times d_{\ell-1}}$ with rows $w_i^\top$; weight matrix $W^{(\ell+1)}\in\R^{d_{\ell+1}\times d_\ell}$; derivative-matching order $k\in\N_0$; derivative feature map $\Phi:\R^{d_{\ell-1}}\to\R^D$ of Eq.~\eqref{eq:Phi-def} with $D=\binom{\din+k}{\din}$.
\Ensure $\mathcal K=\{\kappa_1,\dots,\kappa_n\}\subseteq[d_\ell]$ with \REV{$n\le D$}, matrices $\widetilde W^{(\ell)}\in\R^{n\times d_{\ell-1}}$ and $\widetilde W^{(\ell+1)}\in\R^{d_{\ell+1}\times n}$ such that Eq.~\eqref{eq:V-identity} holds and Eq.~\eqref{eq:tilde-B2} holds.
\State $Z \gets \bigl[\,\Phi(w_1)\,\bigm|\,\Phi(w_2)\,\bigm|\,\cdots\,\bigm|\,\Phi(w_{d_\ell})\,\bigr]\in\R^{D\times d_\ell}$
\State $n \gets \rank(Z)$ \Comment{e.g.\ via SVD or rank-revealing QR}
\If{$n = 0$}
    \State \Return $\mathcal K\gets\varnothing$, $\widetilde W^{(\ell)}\gets\mathbf 0_{0\times d_{\ell-1}}$, $\widetilde W^{(\ell+1)}\gets\mathbf 0_{d_{\ell+1}\times 0}$
\EndIf
\State $(I^\star,\mathcal K^\star)\gets\displaystyle\arg\max_{\substack{I\subseteq[D],\,\mathcal K\subseteq[d_\ell]\\ |I|=|\mathcal K|=n}}\bigl|\det Z[I,\mathcal K]\bigr|$ \Comment{maximum-volume $n\times n$ submatrix}
\State $\mathcal K\gets\mathcal K^\star=\{\kappa_1,\dots,\kappa_n\}$
\State $Z_\mathcal K\gets\bigl[\,\Phi(w_{\kappa_1})\,\bigm|\,\cdots\,\bigm|\,\Phi(w_{\kappa_n})\,\bigr]\in\R^{D\times n}$
\State $C\gets Z_\mathcal K^+\,Z\in\R^{n\times d_\ell}$ \Comment{e.g.\ via thin QR of $Z_\mathcal K$}
\State $\widetilde W^{(\ell+1)}\gets W^{(\ell+1)}\,C^\top\in\R^{d_{\ell+1}\times n}$
\State $\widetilde W^{(\ell)}\gets\bigl[\,w_{\kappa_1}\,\bigm|\,\cdots\,\bigm|\,w_{\kappa_n}\,\bigr]^\top\in\R^{n\times d_{\ell-1}}$ \Comment{rows of $W^{(\ell)}$ indexed by $\mathcal K$}
\State \Return $\mathcal K$, $\widetilde W^{(\ell)}$, $\widetilde W^{(\ell+1)}$
\end{algorithmic}
\end{algorithm}

\begin{proposition}[Correctness of Algorithm~\ref{alg:compression}]
\label{prop:alg-correctness}
The output of Algorithm~\ref{alg:compression} satisfies Eq.~\eqref{eq:V-identity}. In addition,
\begin{equation}
    \max_{j\in[d_{\ell+1}]}
    \|\widetilde W^{(\ell+1)}_{j,:}\|_1
    \le nV_{\max}\le D_kV_{\max}.
\end{equation}
\end{proposition}

\begin{proof}
For $n=0$, $Z=0$ and both sides of Eq.~\eqref{eq:V-identity} vanish. For $n\ge1$, Algorithm~\ref{alg:compression} implements the maximum-volume construction in Lemma~\ref{lem:rank} and the weight definition in Proposition~\ref{prop:compressed-V}; thus Eq.~\eqref{eq:V-identity} holds. The $\ell_1$ bound is exactly Corollary~\ref{cor:l1}.
\end{proof}

\begin{REV*}
Once the derivative-feature matrix $Z$ is available, the arithmetic cost of Algorithm~\ref{alg:compression} simplifies considerably in the rank-saturated case $n=D<d_\ell$. It is
\begin{equation}
\label{eq:alg-cost}
    \underbrace{O\!\left(\binom{d_\ell}{D}D^3\right)}_{\text{maximum-volume search}}
    +
    \underbrace{O\!\left(d_{\ell+1}d_\ell D\right)}_{\text{outgoing-weight update}}.
\end{equation}
Indeed, the rank-factorization and pseudoinverse work is $O(D^2d_\ell)$ and is absorbed by the exhaustive-search term under $D<d_\ell$. The maximum-volume search is combinatorial and infeasible even for modest $D$ and $d_\ell$. It is needed here to obtain the uniform coefficient bound in Eq.~\eqref{eq:C-l1}; practical implementations can instead use polynomial-time rank-revealing column selection. Finding the rank in advance is also unnecessary when one directly targets the upper bound $D$.

The cost of forming $Z$ by evaluating the required derivatives is harder to characterize: it depends on the input-to-layer subnetwork, activation, differentiation method, and implementation-level reuse. We therefore do not assign it an implementation-independent arithmetic bound. Figure~\ref{fig:compression-runtime-scaling} instead reports the measured wall-clock scaling of the complete practical compression pipeline used in our experiments.

\begin{figure}[t]
    \centering
    \includegraphics[width=0.85\textwidth]{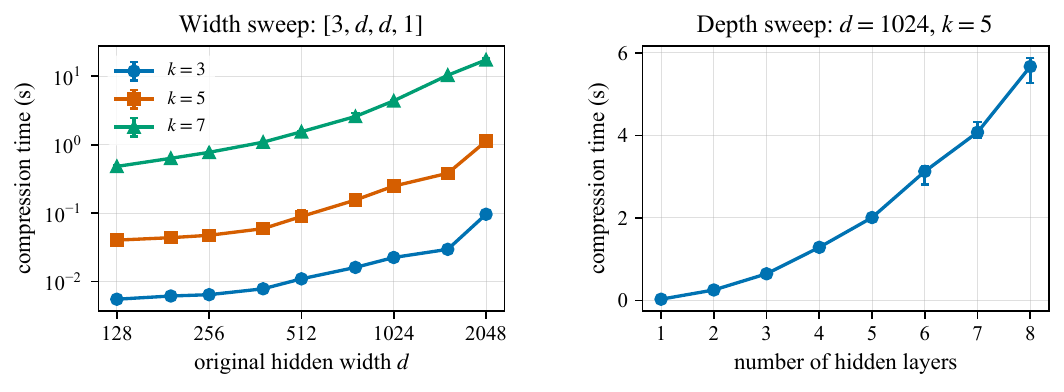}
    \caption{End-to-end runtime of backward compression with the practical pivoted-QR implementation. The timed region includes derivative-feature evaluation, basis selection, the coefficient solve, weight reconstruction for every hidden layer, and the implementation's numerical diagnostics; teacher construction and a small library warm-up are excluded. We use randomly initialized float64 sigmoid networks on CPU with one Torch/BLAS thread. Points are medians of three complete compressions and bars span their minima and maxima. Left: teachers $[3,d,d,1]$ for $k=3,5,7$. Right: teachers $[3,1024,\ldots,1024,1]$ with one through eight hidden layers at $k=5$. }
    \label{fig:compression-runtime-scaling}
\end{figure}
\end{REV*}

\subsection{The main theorem for single-layer compression}

\begin{proof}[\REV{Proof of Theorem~\ref{thm:main}}]
By Proposition~\ref{prop:compressed-V} and Corollary~\ref{cor:l1}, there are compressed matrices with \REV{$n\le D_k$} satisfying Eq.~\eqref{eq:match} and $\max_j\|\widetilde W^{(\ell+1)}_{j,:}\|_1\le D_kV_{\max}$. Theorem~\ref{thm:step1} gives
\begin{equation}
\begin{aligned}
    |y_j(x)-\tilde y_j(x)|
    &\le (1+D_k)C_0e^{-(k+1)\lambda}\\
    &\le 2D_kC_0e^{-(k+1)\lambda},
\end{aligned}
\label{eq:thm-main-bound}
\end{equation}
because $D_k\ge 1$. Since
\begin{equation}
\begin{aligned}
    D_k=\binom{\din+k}{\din}
    &=\frac{\prod_{r=1}^{\din}(k+r)}{\din!}\\
    &\le\frac{(k+\din)^{\din}}{\din!},
\end{aligned}
\label{eq:binom-bound}
\end{equation}
Eq.~\eqref{eq:k-cond} implies the right-hand side is at most $\varepsilon_\ell$, proving Eq.~\eqref{eq:thm-bound}.

For the asymptotics, put $A=\log(2C_0/\varepsilon_\ell)$ and
$h(k)=(k+1)\lambda-\din\log(k+\din)+\log(\din!)-A$. The function $h$ is eventually increasing and tends to infinity, so the minimal feasible $k^\star$ exists. The feasibility inequality implies $k^\star\ge A/\lambda-O(\log A)$, while choosing
$k=\lceil A/\lambda\rceil+\lceil\din\log(A/\lambda+\din)/\lambda\rceil$ is feasible for large $A$. Hence
$k^\star=A/\lambda+O(\log A)$, which is Eq.~\eqref{eq:kstar-asymp}. Substituting this into $D_k\le(k+\din)^{\din}/\din!$ gives Eq.~\eqref{eq:dprime-asymp}.
\end{proof}

\section{Multi-layer composition}
\label{app:composition}

Fix a permutation $(\ell_1,\dots,\ell_{L-1})$ of $\{1,\dots,L-1\}$ (the \emph{compression order}) and per-step budgets $\varepsilon_1,\dots,\varepsilon_{L-1}>0$. Write $F=F^{(0)}$ for the original network; let $F^{(t)}$ be obtained from $F^{(t-1)}$ by applying Algorithm~\ref{alg:compression} to layer $\ell_t$ with matching order $k_t\coloneqq k^\star(\varepsilon_t)$ computed from the constants of $F^{(t-1)}$ via Theorem~\ref{thm:main}, and $\widetilde F\coloneqq F^{(L-1)}$. We write $y^{(\ell+1)}\coloneqq W^{(\ell+1)}h^{(\ell)}$ for the pre-activation entering layer $\ell+1$ (so $y^{(L)}=y_{\mathrm{out}}$, and the downstream map of layer $L-1$ is the identity). Each constant $X^{(\ell)}$ from Section~\ref{sec:setup} carries an extra superscript $X^{(\ell,t)}$ when evaluated for layer $\ell$ in $F^{(t)}$ (so $W_{\max}^{(\ell,t)}$, $V_{\max}^{(\ell+1,t)}$, $\rho^{(\ell,t)}$, $C_0^{(\ell,t)}$, $\lambda^{(\ell,t)}$, where the upper index of $V_{\max}$ is its \emph{matrix} index $\ell+1$ and the others are indexed by the layer being compressed). Throughout this appendix we treat the depth $L$ as a fixed constant; suppressed constants in the deep theorem are allowed to depend on $L$.

\begin{assumption}[Globally Lipschitz $\sigma$]\label{asm:Lsigma}
$\sigma$ is globally $L_\sigma$-Lipschitz on $\R$ for some finite $L_\sigma$.
\end{assumption}

This is genuinely stronger than Assumption~\ref{asm:holo}: it excludes entire activations such as $\exp$, which are admissible in Sections~\ref{sec:setup}--\ref{sec:deep} via $\rho_\sigma=\infty$. Throughout this section we add Assumption~\ref{asm:Lsigma} to the standing assumptions.

\begin{lemma}[Downstream Lipschitz constant]
\label{lem:downstream-lip}
For any $\ell\in\{1,\dots,L-1\}$ and $t\ge0$, the Lipschitz constant $\Lambda_{>\ell}^{(t)}$ of the map $y^{(\ell+1)}\mapsto y_{\mathrm{out}}$ in $F^{(t)}$, in the $\ell_\infty$-to-$\ell_\infty$ sense, satisfies
\begin{equation}
    \Lambda_{>\ell}^{(t)}
    \le
    L_\sigma^{\,L-\ell-1}
    \prod_{j=\ell+2}^{L}\bigl\|W^{(j,t)}\bigr\|_{\mathrm{op},\infty}.
    \label{eq:downstream-lip}
\end{equation}
Here $\bigl\|A\bigr\|_{\mathrm{op},\infty}\coloneqq\sup_{u\neq 0}\|Au\|_\infty/\|u\|_\infty=\max_i\sum_j|A_{ij}|$ is the $\ell_\infty\!\to\!\ell_\infty$ induced (operator) norm, i.e.\ the maximum absolute row sum. In particular $\Lambda_{>L-1}^{(t)}=1$, the downstream map of the last hidden layer being the identity.
\end{lemma}

\begin{proof}
The downstream map alternates $L-\ell-1$ applications of the elementwise $L_\sigma$-Lipschitz activation with the linear maps $W^{(j,t)}$ for $j=\ell+2,\dots,L$. The induced $\ell_\infty$ operator norm of each matrix is $\|W^{(j,t)}\|_{\mathrm{op},\infty}$, so composition yields Eq.~\eqref{eq:downstream-lip}. For $\ell=L-1$ both the activation count and the product are empty, giving $1$.
\end{proof}

\begin{theorem}[Total error of sequential compression]
\label{thm:total-error}
Suppose Theorem~\ref{thm:main} applies at every step $t$ with the constants of $F^{(t-1)}$. Then
\begin{equation}
    \sup_{x\in\Omega}\bigl\|\widetilde F(x)-F(x)\bigr\|_\infty
    \;\le\;
    \sum_{t=1}^{L-1}\Lambda_{>\ell_t}^{(t-1)}\,\varepsilon_t.
    \label{eq:total-err}
\end{equation}
\end{theorem}

\begin{proof}
Telescope $\widetilde F-F=\sum_t\bigl(F^{(t)}-F^{(t-1)}\bigr)$. The upstream map of layer $\ell_t$ is identical in $F^{(t-1)}$ and $F^{(t)}$, so the two networks feed the same pre-activation input to layer $\ell_t$, and Theorem~\ref{thm:main} bounds the resulting perturbation at the pre-activation $y^{(\ell_t+1)}$ by $\varepsilon_t$ uniformly on $\Omega$. The layers downstream of $y^{(\ell_t+1)}$ coincide in $F^{(t-1)}$ and $F^{(t)}$ (so $\Lambda_{>\ell_t}^{(t-1)}=\Lambda_{>\ell_t}^{(t)}$), and Lemma~\ref{lem:downstream-lip} propagates the perturbation by $\Lambda_{>\ell_t}^{(t-1)}$. Summing gives Eq.~\eqref{eq:total-err}.
\end{proof}

\subsection{Forward vs.\ backward order}

\begin{proposition}[Backward order preserves the per-step constants]
\label{prop:backward}
Let $\ell_t=L-t$. For every $t$,
\begin{equation}
    W_{\max}^{(\ell_t,t-1)}=W_{\max}^{(\ell_t,0)},\qquad
    V_{\max}^{(\ell_t+1,t-1)}\le V_{\max}^{(\ell_t+1,0)},\qquad
    \rho^{(\ell_t,t-1)}=\rho^{(\ell_t,0)}.
    \label{eq:back-cnsts}
\end{equation}
Consequently, for a fixed working radius $\rho_0^{(\ell_t)}\in\bigl(R,\rho^{(\ell_t,0)}\bigr)$, one has $C_0^{(\ell_t,t-1)}\le C_0^{(\ell_t,0)}$ and $\lambda^{(\ell_t,t-1)}=\lambda^{(\ell_t,0)}$, and the predicted compressed width is no larger than the original-network prediction.
\end{proposition}

\begin{proof}
At step $t$, previous backward steps only modified layers with index greater than $\ell_t$. Thus $W^{(\ell_t)}$ and the upstream map $f$ entering layer $\ell_t$ are untouched, giving equality of $W_{\max}$, $L_f$, $\rho_\sigma$, and hence of $\rho$; since $\rho_0^{(\ell_t)}$ is held fixed, $\lambda=\log\bigl(\rho_0^{(\ell_t)}/R\bigr)$ is unchanged. The matrix $W^{(\ell_t+1)}$ may have been row-subsetted by the immediately previous step, so its row $\ell_1$-norms cannot increase; this gives the bound on $V_{\max}^{(\ell_t+1,\cdot)}$. The claim for $C_0$ follows from Eq.~\eqref{eq:C0-def}.
\end{proof}

\begin{proposition}[Forward order can inflate incoming weights]
\label{prop:forward}
Let $\ell_t=t$, and let $k_{t-1}^\star$ be the matching order used at step $t-1$. For every $t\ge2$,
\begin{equation}
    W_{\max}^{(t,t-1)}\;\le\;D_{k_{t-1}^\star}\,V_{\max}^{(t,0)},
    \qquad D_{k_{t-1}^\star}=\binom{\din+k_{t-1}^\star}{\din}.
    \label{eq:fwd-blow}
\end{equation}
The right side depends recursively on previous compression widths and therefore on previous analytic constants.
\end{proposition}

\begin{proof}
At step $t-1$, Algorithm~\ref{alg:compression} replaces $W^{(t)}$ by $W^{(t,t-2)}C_{t-1}^\top$. Corollary~\ref{cor:l1} bounds the maximum row $\ell_1$-norm of the result by $D_{k_{t-1}^\star}\,V_{\max}^{(t,t-2)}$. Earlier forward steps do not touch $W^{(t)}$, so $V_{\max}^{(t,t-2)}=V_{\max}^{(t,0)}$. Since $\|u\|_2\le\|u\|_1$, Eq.~\eqref{eq:fwd-blow} follows.
\end{proof}

Two further obstacles distinguish forward order from backward order: (i) the partial map $f$ entering layer $\ell_t$ is the composition of \emph{already-compressed} upstream layers, so the analytic constants $\rho_f^{(\ell_t,t-1)}$, $L_f^{(\ell_t,t-1)}$ entering Eq.~\eqref{eq:rho-eff} differ in general from their original values and must be re-verified; (ii) the inflated $W_{\max}^{(t,t-1)}$ can shrink $\rho^{(t,t-1)}$ via Eq.~\eqref{eq:rho-eff}, hence increase $C_0^{(t,t-1)}$ and $D_{k_t^\star}$. We do not solve this recursion. Heuristically, since $W_{\max}$ enters $\rho$ only through a logarithm, each $D_{k_t^\star}$ is expected to remain polylogarithmic in $1/\varepsilon$, but with constants strictly larger than the original-network values used in Proposition~\ref{prop:backward}. The remainder of this section therefore works exclusively in backward order.

\subsection{Backward compression error bound}

% \paragraph{Why we do not optimize the allocation.}
% The deep theorem is an \emph{existence} statement: it asserts that a same-depth network of polylogarithmic width reproduces $F$ to accuracy $\varepsilon$. For this we need only exhibit \emph{one} feasible budget allocation whose induced widths are polylogarithmic; we do not need the allocation that minimizes the width. This is fortunate, because the compressed width depends on each per-layer budget only through its logarithm (Theorem~\ref{thm:main}), so all balanced allocations agree at leading order (Remark~\ref{rem:robust}); and because the depth $L$ is fixed, a uniform split incurs only an $L$-dependent constant. We therefore use the uniform allocation throughout and make no optimality claim.

\paragraph{An a priori downstream bound.}
In backward order the downstream matrices have \emph{already} been compressed when layer $\ell_t$ is processed, so the realized downstream constant
\begin{equation}
\label{eq:Lambda-star}
    \Lambda^\star\;\coloneqq\;\max_{1\le t\le L-1}\Lambda_{>\ell_t}^{(t-1)}\;\in\;[1,\infty)
\end{equation}
depends on the very widths produced by the budgets, which in turn depend on $\Lambda^\star$ if it is used to set the budgets. % To break this circularity we fix an upper bound computed entirely from the \emph{original} network $F^{(0)}$.

Fix, once and for all, per-layer working radii $\rho_0^{(\ell)}\in\bigl(R,\rho^{(\ell,0)}\bigr)$ and set $\lambda^{(\ell,0)}=\log\bigl(\rho_0^{(\ell)}/R\bigr)$. With
\begin{equation}
\label{eq:Cstar-lstar}
    C^\star\;\coloneqq\;\max_{1\le\ell\le L-1}C_0^{(\ell,0)},\qquad
    \lambda^\star\;\coloneqq\;\min_{1\le\ell\le L-1}\lambda^{(\ell,0)},
\end{equation}
\REV{For $C^\star>0$, define the \emph{surrogate width} for a budget $s>0$ by}
\begin{equation}
\label{eq:Dbar-def}
    \REV{\begin{aligned}
    \overline D(s)
    &\;\coloneqq\;\binom{\din+k^\star(s)}{\din},\\
    k^\star(s)
    &\;\coloneqq\;\min\Bigl\{k\in\N_0:
        (k+1)\lambda^\star\ge\log(2C^\star/s)\\[-0.2em]
    &\hspace{10.3em}
        +\din\log(k+\din)-\log(\din!)\Bigr\}.
    \end{aligned}}
\end{equation}
Because the feasibility condition Eq.~\eqref{eq:k-cond} is monotone in $(C_0,\lambda)$ and $C^\star\ge C_0^{(\ell,t)}$, $\lambda^\star\le\lambda^{(\ell,t)}$ for every layer and step (Proposition~\ref{prop:backward}), $\overline D(s)$ upper-bounds the compressed width produced at \emph{any} step whose budget is $s$. By Eqs.~\eqref{eq:kstar-asymp}--\eqref{eq:dprime-asymp},
\begin{equation}
\label{eq:Dbar-asymp}
    \REV{\overline D(s)\;=\;O\!\left(\frac{1}{\din!}\left(\frac{\log\!\left(2C^\star/s\right)}{\lambda^\star}\right)^{\din}\right)}
    \qquad(s\to 0^+).
\end{equation}

Write $L_\sigma^{+}\coloneqq\max\{1,L_\sigma\}$ and $\Lambda_0^\star\coloneqq\prod_{j=2}^{L}\max\bigl\{1,\|W^{(j,0)}\|_{\mathrm{op},\infty}\bigr\}$, and define
\begin{equation}
\label{eq:psi-def}
    \psi(\Lambda)\;\coloneqq\;\bigl(L_\sigma^{+}\bigr)^{L-2}\,\Lambda_0^\star\,
    \overline D\!\left(\frac{\varepsilon}{(L-1)\,\Lambda}\right)^{\!L-2}.
\end{equation}

\begin{lemma}[A priori downstream bound]
\label{lem:apriori-lambda}
For every finite $L$ there exists $\overline\Lambda\ge 1$ with $\psi(\overline\Lambda)\le\overline\Lambda$. Taking $\overline\Lambda$ to be the least such value, one has, as $\varepsilon\to 0^+$ with $L$ and $\din$ fixed,
\begin{equation}
\label{eq:Lambda-bar-rate}
    \log\overline\Lambda\;=\;O\!\bigl(\din\,L\,\log\log(1/\varepsilon)\bigr).
\end{equation}
\end{lemma}

\begin{proof}
By Eq.~\eqref{eq:Dbar-asymp}, for fixed $\varepsilon$ and $\Lambda\to\infty$,
\(
\overline D\!\bigl(\varepsilon/((L-1)\Lambda)\bigr)=O\!\bigl((\log\Lambda)^{\din}\bigr),
\)
so $\psi(\Lambda)=O\!\bigl((\log\Lambda)^{(L-2)\din}\bigr)=o(\Lambda)$. Hence $\psi(\Lambda)/\Lambda\to 0$, and there is $\overline\Lambda$ with $\psi(\overline\Lambda)\le\overline\Lambda$. (For $L=2$ the exponent is $0$ and $\psi\equiv\Lambda_0^\star$, so $\overline\Lambda=\max\{1,\Lambda_0^\star\}$ works.)

For the rate, it remains to consider $L\ge 3$. Put
\begin{equation}
    A=\log\!\left(\frac{2C^\star(L-1)}{\varepsilon}\right),
    \qquad m=(L-2)\din.
\end{equation}
By Eq.~\eqref{eq:Dbar-asymp}, there is a constant
$B=B\bigl(L,\Lambda_0^\star,L_\sigma,\lambda^\star,C^\star\bigr)$,
independent of $\varepsilon$ and $u$, such that, for all sufficiently small
$\varepsilon$ and every $u\ge 0$,
\begin{equation}
    \log\psi(e^u)\;\le\;B+m\log(A+u).
    \label{eq:bootstrap}
\end{equation}
Choose a constant $B_+\ge\max\{B,0\}$ and set
$u_A=B_++m\log(2A)$. Since $u_A\le A$ for all sufficiently large $A$,
Eq.~\eqref{eq:bootstrap} gives
\begin{equation}
    \log\psi(e^{u_A})
    \;\le\;B+m\log(A+u_A)
    \;\le\;B+m\log(2A)
    \;\le\;u_A.
\end{equation}
Thus $e^{u_A}$ is feasible, and the least feasible value satisfies
\begin{equation}
    \log\overline\Lambda
    \;\le\;u_A
    \;=\;B_++(L-2)\din\log(2A)
    \;=\;O\!\bigl(\din L\log\log(1/\varepsilon)\bigr),
\end{equation}
which proves Eq.~\eqref{eq:Lambda-bar-rate}.
\end{proof}

\paragraph{Feasibility and width.}
We now show that the uniform allocation against $\overline\Lambda$ is feasible and yields polylogarithmic width. 

\begin{proposition}[Feasibility and width of the uniform allocation]
\label{prop:uniform}
Fix $\overline\Lambda$ as in Lemma~\ref{lem:apriori-lambda} and set
\begin{equation}
\label{eq:eps-uniform}
    \varepsilon_t\;=\;\frac{\varepsilon}{(L-1)\,\overline\Lambda}\qquad(t=1,\dots,L-1).
\end{equation}
Compressing in backward order $\ell_t=L-t$ with these budgets, every realized downstream constant obeys $\Lambda_{>\ell_t}^{(t-1)}\le\overline\Lambda$. Consequently the allocation is feasible,
\begin{equation}
\label{eq:feasible}
    \sum_{t=1}^{L-1}\Lambda_{>\ell_t}^{(t-1)}\,\varepsilon_t\;\le\;\varepsilon,
\end{equation}
\REV{and the retained width at step $t$ satisfies}
\begin{equation}
\label{eq:Dt-uniform}
    \REV{\begin{aligned}
    d_{\ell_t}'
    &\le D_{k_t^\star}
    \le\overline D(\varepsilon_t),\\
    \overline D(\varepsilon_t)
    &=O\!\left(\frac{1}{\din!}\left(\frac{\log\!\bigl(2(L-1)\overline\Lambda C^\star/\varepsilon\bigr)}{\lambda^\star}\right)^{\din}\right).
    \end{aligned}}
\end{equation}
\end{proposition}

\begin{proof}
Because the budgets $\varepsilon_t$ are fixed in advance, the width produced at step $t$ depends only on $\varepsilon_t$ and on the realized constants $C_0^{(\ell_t,t-1)},\lambda^{(\ell_t,t-1)}$, which by Proposition~\ref{prop:backward} (with the fixed choice $\rho_0^{(\ell_t)}$) satisfy $C_0^{(\ell_t,t-1)}\le C^\star$ and $\lambda^{(\ell_t,t-1)}\ge\lambda^\star$. Hence $D_{k_t^\star}\le\overline D(\varepsilon_t)$ for every $t$; \REV{Corollary~\ref{cor:l1} also gives $d_{\ell_t}'\le D_{k_t^\star}$, and the asymptotic in Eq.~\eqref{eq:Dt-uniform} is Eq.~\eqref{eq:Dbar-asymp} at $s=\varepsilon_t$.}

For the downstream constant, fix a step $t$ and a downstream matrix index $j\in\{\ell_t+2,\dots,L\}$. In backward order, layer $j$ is compressed at step $L-j<t$ (row-subsetting $W^{(j)}$, which does not increase its maximum row $\ell_1$-norm), and layer $j-1$ is compressed at step $L-j+1\le t-1$ (reweighting $W^{(j)}$). Corollary~\ref{cor:l1} applied at that reweighting gives
\begin{equation}
    \bigl\|W^{(j,t-1)}\bigr\|_{\mathrm{op},\infty}
    \;\le\;
    D_{k_{j-1}^\star}\,\bigl\|W^{(j,0)}\bigr\|_{\mathrm{op},\infty}
    \;\le\;
    \overline D(\varepsilon_t)\,\max\bigl\{1,\|W^{(j,0)}\|_{\mathrm{op},\infty}\bigr\},
\end{equation}
using $D_{k_{j-1}^\star}\le\overline D(\varepsilon_t)$ (the budget is uniform). Substituting into Lemma~\ref{lem:downstream-lip} and bounding $L_\sigma^{\,L-\ell_t-1}\le(L_\sigma^{+})^{L-2}$ and $L-\ell_t-1\le L-2$,
\begin{equation}
    \Lambda_{>\ell_t}^{(t-1)}
    \;\le\;
    (L_\sigma^{+})^{L-2}\,\Lambda_0^\star\,\overline D(\varepsilon_t)^{\,L-2}
    \;=\;\psi(\overline\Lambda)\;\le\;\overline\Lambda,
\end{equation}
the last step by the defining inequality of $\overline\Lambda$. Then $\sum_{t=1}^{L-1}\Lambda_{>\ell_t}^{(t-1)}\varepsilon_t\le\overline\Lambda\,(L-1)\,\varepsilon_t=\varepsilon$, which is Eq.~\eqref{eq:feasible}.
\end{proof}

\begin{remark}[Robustness of the allocation]
\label{rem:robust}
Because $D_{k_t^\star}$ depends on $\varepsilon_t$ only through $\log(1/\varepsilon_t)$, the width bound Eq.~\eqref{eq:Dt-uniform} is unchanged, up to a constant factor, for any allocation $\varepsilon_t=p_t\,\varepsilon/\overline\Lambda$ with $\sum_t p_t=1$ and $p_t\in[c_1/L,c_2/L]$ for absolute $c_1,c_2>0$. One may, for instance, spend a slightly larger budget on layers with larger $C_0^{(\ell,0)}$ or smaller $\lambda^{(\ell,0)}$ without changing the asymptotic compressed width. This allocation-insensitivity (for fixed $L$) is precisely why optimizing the split is unnecessary; the uniform choice Eq.~\eqref{eq:eps-uniform} is simply the most convenient feasible one.
\end{remark}

\paragraph{The deep theorem.}

\begin{theorem}[Polylog deep-MLP compression]
\label{thm:deep-comp}
\REV{Fix one original teacher network $F^{(0)}$ for which Assumptions~\ref{asm:holo}--\ref{asm:R-rho} hold at every hidden layer and Assumption~\ref{asm:Lsigma} holds. Compress in backward order $\ell_t=L-t$ with the uniform allocation Eq.~\eqref{eq:eps-uniform}, where $C^\star,\lambda^\star$ are the original-teacher constants in Eq.~\eqref{eq:Cstar-lstar} and $\overline\Lambda$ is given by Lemma~\ref{lem:apriori-lambda}. Then the compressed network satisfies}
\begin{equation}
\label{eq:total-bound}
    \sup_{x\in\Omega}\bigl\|\widetilde F(x)-F(x)\bigr\|_\infty\;\le\;\varepsilon,
\end{equation}
\REV{and every retained hidden width obeys}
\begin{equation}
\label{eq:deep-width}
    \REV{\begin{aligned}
    d_\ell'
    &\le
    \overline D\!\left(\frac{\varepsilon}{(L-1)\overline\Lambda}\right),\\
    \overline D\!\left(\frac{\varepsilon}{(L-1)\overline\Lambda}\right)
    &=O\!\left(\frac{1}{\din!}\left(\frac{\log\!\bigl(2(L-1)\overline\Lambda C^\star/\varepsilon\bigr)}{\lambda^\star}\right)^{\din}\right),
    \qquad \ell=1,\dots,L-1.
    \end{aligned}}
\end{equation}
\REV{Since $\log\overline\Lambda=O(\din L\log\log(1/\varepsilon))$ by Lemma~\ref{lem:apriori-lambda}, the accuracy-driven cap is $O((\log(1/\varepsilon))^{\din})$ for this fixed teacher network. }
\end{theorem}

\begin{proof}
Backward order leaves the upstream map of layer $\ell_t$ untouched, so Assumptions~\ref{asm:holo} and~\ref{asm:R-rho} continue to hold for $F^{(t-1)}$ at layer $\ell_t$, and Theorem~\ref{thm:main} applies at every step; Proposition~\ref{prop:backward} certifies that its constants are bounded by the original-network values, so $C^\star,\lambda^\star$ are valid uniform bounds. Proposition~\ref{prop:uniform} provides the feasibility bound Eq.~\eqref{eq:feasible}, so Theorem~\ref{thm:total-error} yields Eq.~\eqref{eq:total-bound}; \REV{its retained-width clause Eq.~\eqref{eq:Dt-uniform} gives Eq.~\eqref{eq:deep-width}.} The final estimate substitutes the rate of Lemma~\ref{lem:apriori-lambda}.
\end{proof}

Because the input and output layers are not compressed ($d_0'=\din$, $d_L'=\dout$), the total parameter count obeys
\begin{equation}
\label{eq:param-count}
    \sum_{\ell=1}^{L} d_\ell'\,d_{\ell-1}'
    \;=\;
    O\!\left(
        L\,\Bigl(\max_{1\le\ell\le L-1}d_\ell'\Bigr)^{2}
        +(\din+\dout)\,\max_{1\le\ell\le L-1}d_\ell'
    \right),
\end{equation}
\REV{For fixed $L,\din,\dout$ and the fixed teacher's constants, substituting the accuracy-driven bounds yields a polylogarithmic upper bound in $1/\varepsilon$. Equation~\eqref{eq:param-count} records the explicit dependence on $L$.}

\begin{REV*}
\section{Non-vacuity condition}
\label{app:nonvacuity}

The compression theorem is valid for every covered teacher network, but it certifies a strict width reduction only when the retained dimension is smaller than the original width. This appendix makes that finite-width condition explicit. We write $m$ for the admissible effective input dimension used by the derivative coordinates; the baseline case is $m=\din$. Define
\begin{equation}
    D_k^{(m)}\;\coloneqq\;\binom{m+k}{m}.
\end{equation}

\subsection{Matching order and realized rank}

Before the relaxation used in Eq.~\eqref{eq:k-cond}, the proof of Theorem~\ref{thm:main} gives the error bound
\begin{equation}
    \sup_{x\in\Omega}\|y(x)-\tilde y(x)\|_\infty
    \;\le\;(1+D_k^{(m)})C_0e^{-(k+1)\lambda}.
    \label{eq:nonvac-exact-bound}
\end{equation}
Accordingly, define the sharper finite-scale certification order
\begin{equation}
    k_{\mathrm{cert}}
    \;\coloneqq\;
    \min\left\{k\in\N_0:
    (1+D_k^{(m)})C_0e^{-(k+1)\lambda}\le\varepsilon_\ell
    \right\}.
    \label{eq:kcert}
\end{equation}
The set is nonempty because the exponential term eventually dominates the polynomial $D_k^{(m)}$. The theorem-level $k^\star$ instead uses $2D_k^{(m)}$, followed by the upper bound $D_k^{(m)}\le(k+m)^m/m!$, to obtain the convenient explicit condition Eq.~\eqref{eq:k-cond}. Hence, $k_{\mathrm{cert}}\le k^\star$; the two have the same asymptotic order but differ at constant factor level.

Let
\begin{equation}
    r_k\;\coloneqq\;\rank\Phi_k(W^{(\ell)})
    \;\le\;\min\{D_k^{(m)},d_\ell\}
    \label{eq:realized-rank}
\end{equation}
be the realized derivative-feature rank. The rank-basis construction retains $r_k$ neurons. Thus
\begin{equation}
    D_{k_{\mathrm{cert}}}^{(m)}<d_\ell
    \label{eq:nonvac-dimension}
\end{equation}
is an a priori sufficient condition for strict compression. Condition~\eqref{eq:nonvac-dimension} is not necessary: feature dependencies can make $r_k\ll D_k^{(m)}$. The compression error can remain safe when Eq.~\eqref{eq:nonvac-dimension} fails; what becomes vacuous is only the guaranteed width reduction.

\subsection{A closed-form sufficient threshold}

For $d_\ell\ge2$, a convenient sufficient condition for Eq.~\eqref{eq:nonvac-dimension} is
\begin{equation}
    \boxed{
    \lambda\left[(m!\,d_\ell)^{1/m}-m\right]
    \;\ge\;
    \log\!\left(\frac{(1+d_\ell)C_0}{\varepsilon_\ell}\right)}.
    \label{eq:nonvac-sufficient}
\end{equation}
To verify it, put
\begin{equation}
    a\;\coloneqq\;\frac{1}{\lambda}
    \log\!\left(\frac{(1+d_\ell)C_0}{\varepsilon_\ell}\right).
\end{equation}
If $a>0$, choose $\widehat k=\lceil a\rceil-1$. Equation~\eqref{eq:nonvac-sufficient} gives
\begin{equation}
    \widehat k+m<a+m\le(m!\,d_\ell)^{1/m},
\end{equation}
so
\begin{equation}
    D_{\widehat k}^{(m)}
    =\frac{\prod_{j=1}^m(\widehat k+j)}{m!}
    \le\frac{(\widehat k+m)^m}{m!}<d_\ell.
\end{equation}
Moreover,
\begin{equation}
    (1+D_{\widehat k}^{(m)})C_0e^{-(\widehat k+1)\lambda}
    \le(1+d_\ell)C_0e^{-\lambda\lceil a\rceil}
    \le\varepsilon_\ell.
\end{equation}
If $a\le0$, the direct choice $\widehat k=0$ has $D_0^{(m)}=1<d_\ell$ and also satisfies Eq.~\eqref{eq:nonvac-exact-bound}. Therefore, Eq.~\eqref{eq:nonvac-sufficient} indeed certifies a strictly narrower layer. It is intentionally only sufficient.

% \subsection{Representative finite-width thresholds}

% Table~\ref{tab:nonvac-thresholds} evaluates the exact proof-level criterion for the normalized illustration $C_0=\lambda=1$ and $\varepsilon_\ell=10^{-3}$. Strict dimension-count compression is certified for $d_\ell>D_{k_{\mathrm{cert}}}^{(m)}$; realized rank can lower this threshold. The rapid growth with $m$ explains why the theorem is most informative at low effective dimension. These values are illustrative, because $C_0$ and $\lambda$ are layer- and teacher-specific.

% \begin{table}[h]
%     \color[HTML]{4169E1}
%     \centering
%     \small
%     \begin{tabular}{@{}rrr@{}}
%         \toprule
%         Effective dimension $m$ & $k_{\mathrm{cert}}$ & $D_{k_{\mathrm{cert}}}^{(m)}$ \\
%         \midrule
%         1  & 9  & 10 \\
%         2  & 11 & 78 \\
%         3  & 13 & 560 \\
%         4  & 14 & 3{,}060 \\
%         5  & 16 & 20{,}349 \\
%         6  & 18 & 134{,}596 \\
%         10 & 25 & 183{,}579{,}396 \\
%         16 & 36 & 10{,}363{,}194{,}502{,}115 \\
%         \bottomrule
%     \end{tabular}
%     \caption{Proof-level non-vacuity thresholds for $C_0=\lambda=1$ and $\varepsilon_\ell=10^{-3}$. The a priori strict-compression condition is $d_\ell>D_{k_{\mathrm{cert}}}^{(m)}$.}
%     \label{tab:nonvac-thresholds}
% \end{table}
\end{REV*}

\begin{REV*}
\section{Information-theoretic perspectives}
\label{app:info-theory}

This appendix records information-theoretic evidence for the exponent in our retained-width bound.
% These arguments are deliberately separated from the theorem: they do \emph{not} prove a matching lower bound for compression of an individual teacher network. In particular, the class of functions generated by covered teachers may be much smaller than a full analytic ball, and unrestricted real-valued network parameters need not obey a finite-precision coding model.

\subsection{Kolmogorov-width}

Let $K$ be a compact set strictly inside a complex domain $D\subset\C^m$, and let
\begin{equation}
    \mathcal A_M(D)\;\coloneqq\;
    \{f\in H^\infty(D):\|f\|_{H^\infty(D)}\le M\}\big|_K
\end{equation}
be the restrictions to $K$ of a bounded holomorphic ball. For standard nested-domain regularity conditions, the Kolmogorov widths of this restriction class satisfy
\begin{equation}
    d_N\bigl(\mathcal A_M(D);C(K)\bigr)
    \;=\;M\exp\!\bigl(-\Theta(N^{1/m})\bigr),
    \label{eq:analytic-nwidth}
\end{equation}
with domain-dependent constants; see the general theory of $n$-widths \citep{pinkus1985nwidths} and sharp holomorphic restriction asymptotics \citep{bandtlow2022width}. Inverting Eq.~\eqref{eq:analytic-nwidth} gives the benchmark dimension
\begin{equation}
    N_\varepsilon
    \;=\;\Theta\!\left(\bigl[\log(M/\varepsilon)\bigr]^m\right).
    \label{eq:analytic-width-inverted}
\end{equation}
Thus the same logarithmic power as our derivative-feature count appears for a canonical analytic function class. This comparison indicates that the exponent $m$ is natural rather than a peculiarity of the rank-reduction algorithm. It is not a compression lower bound for one fixed teacher: obtaining such a result would require showing that the relevant teacher-generated class contains a comparably large analytic ball.

\subsection{Metric-entropy}

For prototypical analytic classes on nested cubes or polydisks, classical $\varepsilon$-entropy estimates have the form
\begin{equation}
    \log \mathcal N\bigl(\varepsilon,\mathcal A_M(D),\|\cdot\|_\infty\bigr)
    \;=\;\Theta\!\left(\bigl[\log(M/\varepsilon)\bigr]^{m+1}\right)
    \label{eq:analytic-entropy}
\end{equation}
under the corresponding geometric regularity conditions \citep{kolmogorov1961entropy}. A stable model with $P$ bounded-precision real parameters has, heuristically, covering entropy of order at most $P\log(1/\varepsilon)$. Comparing with Eq.~\eqref{eq:analytic-entropy} suggests $P\gtrsim[\log(1/\varepsilon)]^m$. The word \emph{heuristically} is essential: without bounded ranges, finite precision, or regularity of the parameter-to-function map, a raw parameter count is not an information-theoretic code-length bound. The entropy comparison is therefore supporting evidence, not a theorem about arbitrary neural parameterizations.

\subsection{Why no activation-uniform lower bound is claimed}

There is a complementary obstruction to a universal width lower bound. \citet{maiorov1999lower} construct a particular analytic, strictly monotone sigmoidal activation for which two-hidden-layer networks with fixed finite hidden widths are dense in continuous functions on compacta. Consequently, no lower bound based only on hidden width can hold uniformly over all analytic sigmoidal activations. A matching lower bound for the present setting would need additional restrictions---for example, a fixed nonpathological activation together with quantitative bounds on weights, precision, or parameter stability.
\end{REV*}

\begin{REV*}
\section{Consistency with classical approximation rates}
\label{app:classical-rates}

There is no contradiction between our polylogarithmic fixed-teacher statement and classical polynomial approximation rates. Our theorem compresses one fixed teacher network, for which $C_0$ is a fixed constant. Universal approximation and minimax results instead concern an accuracy-indexed family: the approximating network $\mathcal N_1(\varepsilon)$ changes as $\varepsilon$ changes. The following example shows explicitly that our finite weight range assumptions cannot hold uniformly along such a family, whereas they naturally hold for any fixed network.

\subsection{A concrete example}

Consider scalar-output bilayer MLPs. In the bias-free notation, the derivative features $\partial_x^\alpha g(0,w)$ are weighted degree-$|\alpha|$ monomials in $w$, so matching all derivatives through order $k$ reduces to matching the corresponding polynomial moments of the neuron weights. Let $\mathcal P_k$ be the linear space of polynomials in $x$ of total degree at most $k$, and define
\begin{equation}
    E_k(h)\;\coloneqq\;
    \inf_{p\in\mathcal P_k}\|h-p\|_{L^\infty(\Omega)}.
\end{equation}
For a one-dimensional teacher output $y$, the Taylor-remainder argument in the proof of Theorem~\ref{thm:step1} gives
\begin{equation}
    E_k(y)\;\le\;C_0e^{-\lambda(k+1)},
    \qquad
    C_0\;\coloneqq\;V_{\max}M_g\frac{\rho_0}{\rho_0-R},
    \qquad
    \lambda\;\coloneqq\;\log(\rho_0/R).
    \label{eq:family-geometric}
\end{equation}

Fix an integer $n\ge1$ and
\begin{equation}
    \Omega\;\coloneqq\;[-1,1]^{\din},
    \qquad R=\sqrt{\din},
    \qquad
    f(x)\;\coloneqq\;\frac{1}{n!}(x_1)_+^n,
    \label{eq:bernstein-target}
\end{equation}
where $(t)_+^n=t^n$ for $t>0$ and is zero otherwise. Three properties are relevant: $0\in\Omega$; the target lies in the unit ball of $W^{n,\infty}(\Omega)$ because its weak derivative $\partial_1^nf=\mathbf 1_{\{x_1>0\}}$ is bounded, while $f\notin C^n$; and the singular set $\{x_1=0\}$ meets the interior of $\Omega$. Bernstein's saturation theorem gives
\begin{equation}
\begin{aligned}
    E_k(f)
    &\ge
    \inf_{p\in\mathcal P_k^{(1\mathrm D)}}
    \left\|\frac{1}{n!}(t)_+^n-p(t)\right\|_{L^\infty([-1,1])}\\
    &\ge c_nk^{-n},
\end{aligned}
    \label{eq:bernstein-lower}
\end{equation}
for a Bernstein saturation constant $c_n>0$ \citep{devore1993constructive}, e.g., $c_1\approx0.298$ and $c_2\approx0.188$.

\subsection{Finding a universal approximation family \texorpdfstring{$\mathcal N_1(\varepsilon)$}{N1(epsilon)}}

\paragraph{Step 0---Assumption~\ref{asm:R-rho} pins the input weights.}
At the first hidden layer the upstream map is $f_{\mathrm{in}}=\mathrm{id}$, hence $L_f\equiv1$ and $\rho=\rho_\sigma/W_{\max}^{(1)}$. Assumption~\ref{asm:R-rho} is therefore equivalent to
\begin{equation}
    W_{\max}^{(1)}
    <\frac{\rho_\sigma}{R}
    =\frac{\rho_\sigma}{\sqrt{\din}}
    \eqqcolon W.
    \label{eq:family-input-cap}
\end{equation}

\paragraph{Step 1---the family exists.}
Because Eq.~\eqref{eq:family-input-cap} caps the input weights, we need universality inside a bounded-input-weight class. For a superanalytic activation to which their result applies, Theorem~2.6 of \citet{stinchcombe1990bounded} gives density of
\begin{equation}
    \operatorname{span}
    \bigl\{\sigma(w\cdot x+b):\|(w,b)\|_\infty\le B\bigr\}
\end{equation}
in $C(\Omega)$; Corollary~3.5 of \citet{pinkus1999approximation} sharpens this to parameters in an arbitrarily small neighborhood of the origin. Choose once and for all $W_0<W$ and such a fixed neighborhood. Then, for every $\varepsilon>0$, there is a finite network
\begin{equation}
    \mathcal N_1(\varepsilon)(x)
    =\sum_{i=1}^{q(\varepsilon)}
    v_i(\varepsilon)\,
    \sigma\!\left(w_i(\varepsilon)^\top x+b_i(\varepsilon)\right),
    \qquad
    \|w_i(\varepsilon)\|\le W_0,
    \label{eq:universal-family}
\end{equation}
such that
\begin{equation}
    \|\mathcal N_1(\varepsilon)-f\|_{L^\infty(\Omega)}\le\varepsilon;
\end{equation}
here $q(\varepsilon)$ is the number of hidden units and is free to grow with accuracy.

\paragraph{Step 2---which constants move with $\varepsilon$.}
The fixed margin $W_0<W$ makes the effective radius $\rho[\varepsilon]$ uniformly larger than $R$. We may therefore choose one $\rho_0>R$ for the whole family, so $\lambda=\log(\rho_0/R)>0$ is fixed; the bounded incoming affine parameters also give a uniform upper bound on $M_g[\varepsilon]$. The remaining free factor in Eq.~\eqref{eq:family-geometric} is
\begin{equation}
    V_{\max}[\varepsilon]
    =\sum_{i=1}^{q(\varepsilon)}|v_i(\varepsilon)|.
\end{equation}
A uniform bound on $V_{\max}[\varepsilon]$ would give $C_0=O(1)$ and hence the polylogarithmic retained width uniformly over this family. Even polynomial growth in $1/\varepsilon$ would leave $\log C_0=O(\log(1/\varepsilon))$. We next show that this family cannot have either behavior: the uniform-in-$\varepsilon$ form of Assumption~\ref{asm:B2} must fail much more decisively.

\subsection{Why \texorpdfstring{$V_{\max}$}{Vmax} must diverge}

Combining Eq.~\eqref{eq:bernstein-lower}, the approximation property of $\mathcal N_1(\varepsilon)$, and Eq.~\eqref{eq:family-geometric} gives, for all sufficiently large integers $k$,
\begin{equation}
\begin{aligned}
    c_nk^{-n}
    &\le E_k(f)\\
    &\le\|f-\mathcal N_1(\varepsilon)\|_\infty
       +E_k(\mathcal N_1(\varepsilon))\\
    &\le\varepsilon
       +C_0[\mathcal N_1(\varepsilon)]e^{-\lambda(k+1)}.
\end{aligned}
    \label{eq:bernstein-family-chain}
\end{equation}
For sufficiently small $\varepsilon$, take
\begin{equation}
    k_\varepsilon
    \;\coloneqq\;
    \left\lfloor\left(\frac{c_n}{2\varepsilon}\right)^{1/n}\right\rfloor.
\end{equation}
Then $c_nk_\varepsilon^{-n}\ge2\varepsilon$ and $k_\varepsilon+1>(c_n/(2\varepsilon))^{1/n}$. Equation~\eqref{eq:bernstein-family-chain} yields
\begin{equation}
\begin{aligned}
    C_0[\mathcal N_1(\varepsilon)]
    &\ge\varepsilon e^{\lambda(k_\varepsilon+1)}\\
    &\ge\varepsilon\exp\!\left[
       \lambda\left(\frac{c_n}{2\varepsilon}\right)^{1/n}
       \right].
\end{aligned}
    \label{eq:C0-must-grow}
\end{equation}
Consequently,
\begin{equation}
    \log C_0[\mathcal N_1(\varepsilon)]
    =\Omega(\varepsilon^{-1/n}).
\end{equation}
Every factor of $C_0$ other than $V_{\max}$ has a uniform upper bound in this construction. Hence
\begin{equation}
    \log V_{\max}[\varepsilon]
    =\Omega(\varepsilon^{-1/n}),
    \label{eq:Vmax-must-grow}
\end{equation}
so the outgoing-mass bound cannot hold uniformly along the accuracy-indexed family, although it remains finite for each fixed teacher.

\subsection{Consequence: no contradiction, but agreement}

If the derivative-matching certificate is required to approximate $\mathcal N_1(\varepsilon)$ by a compressed network $\mathcal N_2(\varepsilon)$ to error at most $\varepsilon$, its matching order must in particular make $C_0[\mathcal N_1(\varepsilon)]e^{-\lambda(k+1)}\le\varepsilon$. Equation~\eqref{eq:C0-must-grow} therefore forces
\begin{equation}
    k\ge\left(\frac{c_n}{2\varepsilon}\right)^{1/n}+O(1),
\end{equation}
and the retained derivative dimension obeys
\begin{equation}
    D_k=\binom{\din+k}{\din}
    =\Omega(\varepsilon^{-\din/n}).
    \label{eq:classical-rate-recovered}
\end{equation}
Thus the same $\varepsilon$-exponent as the classical Sobolev rate reappears once the accuracy dependence of the teacher family is included.

Finally, we remark that our Assumptions~\ref{asm:Wmax} and \ref{asm:B2}, limiting $V_{\max}, W_{\max}< \infty$, naturally stand for one fixed network. Even if a family of networks is concerned, they arise naturally due to the weight regularization terms we typically have in training.
\end{REV*}

\begin{REV*}
\section{Experimental details}
\label{app:experiments}

We use the target in Eq.~\eqref{eq:trained-target} on the Euclidean unit disk. Each teacher network is trained by unconstrained Adam on the same $20{,}000$ Sobol inputs for seeds $0,\ldots,4$. Compression is then applied to the frozen trained weights for $k=0,\ldots,10$, with no fine-tuning or seed selection after compression.
% For comparison, we record the shallow architecture $[2,1024,1]$ in Table~\ref{tab:trained-shallow}.
The architecture in Fig.~\ref{fig:trained-deep} is $[2,512,512,512,1]$, and all three hidden layers are compressed in backward order to the same $D_k=\binom{k+2}{2}$. Each frozen random initialization is an architecture- and seed-matched control.

Task error is MSE against $q$ on $150{,}000$ held-out Sobol points. Function fidelity is the largest teacher--compression discrepancy on those points or in projected local searches initialized at the largest sampled discrepancies.

For the three-hidden-layer teachers, layer 1 has $\rho_1=\pi/W_{\max}^{(1)}>1$ for every seed and therefore passes the radius check. At layer 2 we first compute the exact distance to the nearest upstream sigmoid pole,
\begin{equation}
    \rho_f^{(2)}
    =\min_j
    \frac{\sqrt{(b_j^{(1)})^2+\pi^2}}
         {\|W_{j,:}^{(1)}\|_2},
\end{equation}
and then sample the complex-input Jacobian norm. At $R=1$, define
\begin{equation}
    \widehat q_2
    =\frac{R W_{\max}^{(2)}\widehat L_{h^{(1)}}(R)}{\pi}
\end{equation}
For every seed, $\widehat q_2\ge1$, directly witnessing failure of Assumption~\ref{asm:R-rho}. More quantitatively, the first sampled violating radii for seeds $0,\ldots,4$ are $(0.300,0.287,0.248,0.290,0.261)$. Since $\widehat L_{h^{(1)}}(r)\le L_{h^{(1)}}(r)$, these give $\rho_2$ upper bounds at the same values and hence $R/\rho_2\ge(3.33,3.48,4.03,3.45,3.83)$, respectively.

\begin{table}[H]
    \centering
    \small
    \resizebox{0.94\textwidth}{!}{%
    \begin{tabular}{@{}rrrrr@{}}
        \toprule
        Seed & Teacher MSE & $\rho_1$ & $\widehat q_2$ & First task-preserving $D_k$ (reduction) \\
        \midrule
        0 & $2.3490\!\times\!10^{-4}$ & 1.2632 & 7.87  & 66 ($57.9\times$) \\
        1 & $3.6973\!\times\!10^{-4}$ & 1.2829 & 7.00  & 21 ($522.7\times$) \\
        2 & $2.4447\!\times\!10^{-4}$ & 1.0715 & 14.46 & 36 ($187.7\times$) \\
        3 & $2.3616\!\times\!10^{-4}$ & 1.2011 & 8.32  & 28 ($303.6\times$) \\
        4 & $3.0791\!\times\!10^{-4}$ & 1.0520 & 13.81 & 21 ($522.7\times$) \\
        \bottomrule
    \end{tabular}}
    \caption{Per-seed diagnostics after backward compression of all three hidden layers. The original network has $527{,}361$ parameters. Layer 1 passes the radius check, while the sampled layer-2 diagnostic fails for every seed.}
    \label{tab:trained-deep}
\end{table}

The three-hidden-layer teachers have median MSE $2.445\times10^{-4}$, median width-66 fidelity error $2.305\times10^{-3}$, and random-control median $2.60\times10^{-11}$. Test/train MSE ratios range from $0.99953$ to $0.99980$, and validation/train ratios from $0.99911$ to $1.00342$, so there is no generalization gap indicative of overfitting. Seeds 3 and 4 select step $2250$ because both training and validation errors rise at the final checkpoint; this synchronized fluctuation is optimization instability rather than overfitting.

\end{REV*}

% \newpage
% \input{checklist.tex}

\end{document}